\documentclass[dvipsnames, a4paper, 11pt]{article}

\usepackage[utf8]{inputenc}
\usepackage[T1]{fontenc}
\usepackage{amsmath}
\usepackage{amsthm}
\usepackage{amsfonts}
\usepackage{amssymb}
\usepackage{graphicx,caption,subcaption}
\usepackage{fancybox}
\usepackage{enumitem}
\usepackage{soul}
\usepackage{mathtools}
\usepackage{lmodern}
\usepackage{stmaryrd}
\usepackage{multirow}
\usepackage{multicol}
\usepackage{xcolor}
\usepackage{tabularx}
\usepackage{array}
\usepackage{colortbl}
\usepackage{varwidth}
\PassOptionsToPackage{dvipsnames,svgnames}{xcolor}     
\usepackage{wrapfig}
\usepackage[explicit,calcwidth]{titlesec}
\usepackage{fancyhdr}
\usepackage[left=2.1cm, right=2.1cm, top=2cm, bottom=2cm]{geometry}
\usepackage[skip=8pt plus1pt, indent=0pt]{parskip}
\usepackage{esint}
\usepackage{tocloft}
\usepackage{hyperref}
\usepackage{twoopt}
\usepackage{authblk}
\hypersetup{
    colorlinks=true, %
    linktoc=all,     %
    linkcolor=blue,  %
}
\PassOptionsToPackage{unicode}{hyperref}
\PassOptionsToPackage{naturalnames}{hyperref}
\usepackage{float}

\usepackage{natbib}
\setcitestyle{authoryear,round,citesep={;},aysep={,},yysep={;}}

\usepackage[ruled,vlined,linesnumbered]{algorithm2e}
\labelformat{algocf}{Algorithm #1}

\usepackage{todonotes}
\setuptodonotes{fancyline, color=red!20, shadow, inline}

\usepackage{silence}

\usepackage[bbgreekl]{mathbbol}
\DeclareSymbolFont{bbold}{U}{bbold}{m}{n}
\DeclareSymbolFontAlphabet{\mathbbold}{bbold}
\DeclareSymbolFontAlphabet{\mathbb}{AMSb}%

\newtheorem{theorem}{Theorem}
\newtheorem*{theorem*}{Theorem}

\newtheorem{remark}{Remark}
\newtheorem{prop}{Proposition}
\newtheorem{assumption}{Assumption}
\newtheorem{condition}{Condition}

\labelformat{equation}{(#1)}
\labelformat{figure}{Figure~#1}
\labelformat{section}{Section~#1}
\labelformat{subsection}{Section~#1}
\labelformat{subsubsection}{Section~#1}
\labelformat{theorem}{Theorem~#1}
\labelformat{prop}{Proposition~#1}
\labelformat{lemma}{Lemma~#1}
\labelformat{remark}{Remark~#1}
\labelformat{assumption}{Assumption~#1}
\labelformat{condition}{Condition~#1}
\labelformat{corollary}{Corollary~#1}
\labelformat{algorithm}{Algorithm~#1}
\labelformat{table}{Table~#1}

\newcommand{\N}[0]{\mathbb{N}}

\newcommand{\R}[0]{\mathbb{R}}

\newcommand{\Int}[2]{\displaystyle\int_{#1}^{#2}}

\newcommand{\Sum}[2]{\displaystyle\sum\limits_{#1}^{#2}}

\newcommand{\Inter}[2]{\displaystyle\bigcap\limits_{#1}^{#2}}
\newcommand{\Reu}[2]{\displaystyle\bigcup\limits_{#1}^{#2}}

\newcommand{\dr}[3]{\cfrac{\partial ^{#2} {#3}}{\partial {#1}^{#2}}}

\newcommand{\dd}[0]{\mathrm{d}}
\definecolor{mybluei}{RGB}{0,173,239}

\newcommand{\red}[1]{#1}
\newcommand{\redtwo}[1]{#1}%

\renewcommand{\epsilon}{\varepsilon}

\newcommand{\nt}[1]{{\left\vert\kern-0.25ex\left\vert\kern-0.25ex\left\vert #1 
    \right\vert\kern-0.25ex\right\vert\kern-0.25ex\right\vert}}

\DeclareMathOperator{\sign}{sign}

\newcommand{\W}[0]{\mathrm{W}}
\renewcommand{\SS}{\mathbb{S}}
\newcommand{\opn}[1]{\left\|#1\right\|_{\mathrm{op}}}

\newcommand{\app}[4]{\left\lbrace\begin{array}{ccc}
   #1 & \longrightarrow & #2 \\
   #3 & \longmapsto & #4 \\
\end{array} \right.}

\newcommand{\eqlabel}[1]{\tag{#1}\label{eqn:#1}}

\newcommand{\oll}[1]{\overline{#1}}
\newcommand{\ull}[1]{\underline{#1}}

\newcommand{\npoints}{n}
\newcommand{\SWY}{\mathcal{E}}

\newcommand{\SW}{\mathrm{SW}}
\newcommand{\sort}[2]{\tau_{#1}^{#2}}  %

\newcommand{\lr}{\alpha}  %
\newcommand{\noise}{a}  %

\newcommand{\mx}{\mathbbold{x}}
\newcommand{\mxn}{\mx^{\otimes \npoints}}
\newcommand{\mxN}{\mx^{\otimes \N}}
\newcommand{\X}{\mathcal{X}}
\newcommand{\Xn}{\X^n}
\newcommand{\dimx}{{d_x}}

\newcommand{\my}{\mathbbold{y}}
\newcommand{\myn}{\my^{\otimes \npoints}}
\newcommand{\myN}{\my^{\otimes \N}}
\newcommand{\Y}{\mathcal{Y}}
\newcommand{\Yn}{\Y^n}
\newcommand{\dimy}{{d_y}}

\newcommand{\data}{z}
\newcommand{\mdata}{\mathbbold{\data}}
\newcommand{\Data}{\mathcal{Z}}

\newcommand{\p}{u}
\newcommand{\dimp}{{d_{\p}}}
\newcommand{\mpzero}{\mathbbold{\p}_0}
\renewcommand{\mp}{\mathbbold{\p}}
\newcommand{\U}{\mathcal{U}}

\newcommand{\lipT}{L}
\newcommand{\MddT}{M}
\newcommand{\classT}{\mathcal{T}}
\newcommand{\nlayers}{N}
\newcommand{\layerout}{h}
\newcommand{\activ}{a}

\newcommand{\linunit}{A}
\newcommand{\intercept}{B}

\newcommand{\mnoise}{\bbespilon}

\title{Convergence of SGD for Training Neural Networks with Sliced Wasserstein Losses}
\author[1]{Eloi Tanguy}
\affil[1]{Universit\'e Paris Cit\'e, CNRS, MAP5, F-75006 Paris, France}
\date{October 2023}
  
\begin{document}
\maketitle

\begin{abstract} 
	Optimal Transport has sparked vivid interest in recent years, in particular thanks to the Wasserstein distance, which provides a geometrically sensible and intuitive way of comparing probability measures. For computational reasons, the Sliced Wasserstein (SW) distance was introduced as an alternative to the Wasserstein distance, and has seen uses for training generative Neural Networks (NNs). While convergence of Stochastic Gradient Descent (SGD) has been observed practically in such a setting, there is to our knowledge no theoretical guarantee for this observation. Leveraging recent works on convergence of SGD on non-smooth and non-convex functions by \citet{bianchi2022convergence}, we aim to bridge that knowledge gap, and provide a realistic context under which fixed-step SGD trajectories for the SW loss on NN parameters converge. More precisely, we show that the trajectories approach the set of (sub)-gradient flow equations as the step decreases. Under stricter assumptions, we show a much stronger convergence result for noised and projected SGD schemes, namely that the long-run limits of the trajectories approach a set of generalised critical points of the loss function.
\end{abstract}

\tableofcontents

\section{Introduction}
\subsection{Optimal Transport in Machine Learning}
Optimal Transport (OT) allows the comparison of measures on a metric space by generalising the use of the ground metric. Typical applications use the so-called 2-Wasserstein distance, defined as
\begin{equation}\eqlabel{W2}
	\forall \mx, \my \in \mathcal{P}_2(\R^d),\; \W_2^2(\mx, \my) := \underset{\bbpi \in \Pi(\mx, \my)}{\inf}\Int{\R^d\times\R^d}{}\|x-y\|_2^2\dd \bbpi(x, y),
\end{equation}
where $\mathcal{P}_2(\R^d)$ is the set of probability measures on $\R^d$ admitting a second-order moment and where $\Pi(\mx, \my)$ is the set of measures of $\mathcal{P}_2(\R^d \times \R^d)$ of first marginal $\mx$ and second marginal $\my$. One may find a thorough presentation of its properties in classical monographs such as \citet{computational_ot, santambrogio2015optimal, villani}

The ability to compare probability measures is useful in probability density fitting problems, which are a sub-genre of generation tasks. In this formalism, one considers a probability measure parametrised by a vector $\p$ which is designed to approach a target data distribution $\my$ (typically the real-world dataset). In order to determine suitable parameters, one may choose any probability discrepancy (Kullback-Leibler, Ciszar divergences, f-divergences or Maximum Mean Discrepancy \citep{gretton2006kernel}), or in our case, the Wasserstein distance. In the case of Generative Adversarial Networks, the optimisation problem which trains the "Wasserstein GAN" \citep{pmlr-v70-arjovsky17a} stems from the Kantorovitch-Rubinstein dual expression of the 1-Wasserstein distance.
\subsection{The Sliced Wasserstein Distance as an Alternative}
\redtwo{The Wasserstein distance suffers from the curse of dimensionality, in the sense that the sample complexity for $n$ samples in dimension $d$ is of the order $\mathcal{O}(n^{1/d})$ \citep{dudley1969speed}. Due to this practical limitation and to the computational cost of the Wasserstein distance, the study of cheaper alternatives has become a prominent field of research. A prominent example is Entropic OT introduced by \citet{cuturi2013sinkhorn}, which adds an entropic regularisation term, advantageously making the problem strongly convex. Sample complexity bounds have been derived by \citet{genevay2019sample}, showing a convergence in $\mathcal{O}(\sqrt{n})$ with a constant depending on the regularisation factor.}

Another alternative is the Sliced Wasserstein (SW) Distance introduced by \citet{Rabin_texture_mixing_sw}, which consists in computing the 1D Wasserstein distances between projections of input measures, and averaging over the projections. The aforementioned projection of a measure $\mx$ on $\R^d$ is done by the \textit{push-forward} operation by the map $P_\theta: x \longmapsto \theta^\top  x$. Formally, $P_\theta\#\mx$ is the measure on $\R$ such that for any Borel set $B\subset \R$, $P_\theta\#\mx(B) = \mx(P_\theta^{-1}(B))$. Once the measures are projected onto a line $\R\theta$, the computation of the Wasserstein distance becomes substantially simpler numerically. We illustrate this fact in the discrete case, which arises in practical optimisation settings. Let two discrete measures on $\R^d$: $\bbgamma_X := \frac{1}{\npoints}\sum_k\bbdelta_{x_k},\; \bbgamma_Y := \frac{1}{\npoints}\sum_k\bbdelta_{y_k}$ with supports $X = (x_1, \cdots, x_\npoints)$ and $Y=(y_1, \cdots, y_\npoints) \in \R^{\npoints\times d}$. Their push-forwards by $P_\theta$ are simply computed by the formula $P_\theta \#\bbgamma_X = \frac{1}{n}\sum_k \bbdelta_{P_\theta(x_k)}$, and the 2-Wasserstein distance between their projections can be computed by sorting their supports: let $\sigma$ a permutation sorting $(\theta^\top x_1, \cdots, \theta^\top x_\npoints)$, and $\tau$ a permutation sorting $(\theta^\top y_1, \cdots, \theta^\top y_\npoints)$, one has the simple expression
\begin{equation}
	\W_2^2(P_\theta\#\bbgamma_X, P_\theta\#\bbgamma_Y) = \cfrac{1}{\npoints}\Sum{k=1}{\npoints}(\theta^\top x_{\sigma(k)} - \theta^\top y_{\tau(k)})^2.
\end{equation}
The SW distance is the expectation of this quantity with respect to $\theta\sim \bbsigma$, i.e. uniform on the sphere: $\SW_2^2(\bbgamma_X, \bbgamma_Y) = \mathbb{E}_{\theta \sim \bbsigma}\left[\W_2^2(P_\theta\#\bbgamma_X, P_\theta\#\bbgamma_Y)\right]$. The 2-SW distance is also defined more generally between two measures $\mx, \my \in \mathcal{P}_2(\R^d)$:
\begin{equation}\eqlabel{SW}
	\mathrm{SW}_2^2(\mx, \my) := \Int{\theta \in \SS^{d-1}}{}\W_2^2(P_\theta\#\mx, P_\theta\#\my)\dd \bbsigma(\theta).
\end{equation}
\redtwo{In addition to its computational accessibility, the SW distance enjoys a dimension-free sample complexity \citep{nadjahi_statistical_properties_sliced}. Additional statistical, computational and robustness properties of SW have been explored by \citet{nietert2022statistical}. Moreover, central-limit results have been shown by \citet{xu2022central} for 1-SW and the 1-max-SW distance (a variant of SW introduced by \citet{deshpande2019max}), and related work by \citet{xi2022distributional} shows the convergence of the sliced error process $\theta \longmapsto \sqrt{n}\left(\W_p^p(P_\theta\#\bbgamma_X, P_\theta\#\bbgamma_Y) - \W_p^p(P_\theta\#\mx, P_\theta\#\my)\right)$, where the samples $X\sim\mxn$ and $Y\sim\myn$ are drawn for each $\theta$. Another salient field of research for SW is its metric properties, and while it has been shown to be weaker than the Wasserstein distance in general by \citet{bonnotte}, and metric comparisons with Wasserstein and max-SW have been undergone by \citet{bayraktar2021strong} and \citet{paty2019subspace}.}
\redtwo{\subsection{Related Works}}
Our subject of interest is the theoretical properties of SW as a loss for implicit generative modelling, which leads to minimising $\mathrm{SW}_2^2(T_\p\#\mx, \my)$ in the parameters $\p$, where $\my$ is the target distribution, and $T_\p\#\mx$ is the image by the NN\footnote{Similarly to the 1D case, $T_\p\#\mx$ is the push-forward measure of $\mx$ by $T_\p$, i.e. the law of $T_\p(x)$ when $x\sim \mx$.} of $\mx$, a low-dimensional input distribution (often chosen as Gaussian or uniform noise). In order to train a NN in this manner, at each iteration one draws $\npoints$ samples from $\mx$ and $\my$ (denoted $\bbgamma_X$ and $\bbgamma_Y$ as discrete measures with $\npoints$ points), as well as a projection $\theta$ (or a batch of projections) and performs an SGD step on the sample loss
\begin{equation}\label{eqn:Loss}
	\mathcal{L}(\p) = \mathrm{SW}_2^2(P_\theta\#T_\p\#\bbgamma_X, P_\theta\#\bbgamma_Y) = \cfrac{1}{\npoints}\Sum{k=1}{\npoints}(\theta^\top T_\p(x_{\sigma(k)}) - \theta^\top y_{\tau(k)})^2.
\end{equation}
Taking the expectation of this loss over the samples yields the minibatch Sliced-Wasserstein discrepancy, a member of the minibatch variants of the OT distances, introduced formally by Fatras et al. \cite{fatras2021minibatch}. The framework \ref{eqn:Loss} fits several Machine Learning applications, for instance, \citet{deshpande_generative_sw} trains GANs and auto-encoders with this method, and \citet{Wu2019_SWAE} consider related dual formulations. Other examples within this formalism include the synthesis of images by minimising the SW distance between features of the optimised image and a target image, as done by \citet{heitz2021sliced} for textures with neural features, and by \citet{Tartavel2016} with wavelet features (amongst other methods).

The general study of convergence of SGD in the context of non-smooth, non-convex functions (as is the case of $\mathcal{L}$ from \ref{eqn:Loss}) is an active field of research: \citet{majewski2018analysis} and \citet{davis2020stochastic} show the convergence of diminishing-step SGD under regularity constraints, while \citet{bolte2021conservative} leverage conservative field theory to show convergence results for training with back-propagation. Finally, the recent work by \citet{bianchi2022convergence} shows the convergence of fixed-step SGD schemes on a general function $F$ under weaker regularity assumptions.

More specifically, the study of convergence for OT-based generative NNs has been tackled by \citet{fatras2021minibatch}, who prove strong convergence results for minibatch variants of classical OT distances, namely the Wasserstein distance, the Entropic OT and the Gromov Wasserstein distance (another OT variant introduced by \citet{memoli2011gromov}). A related study on GANs by \citet{huang2023characterizing} derive optimisation properties for one layer and one dimensional Wasserstein-GANs and generalise to higher dimensions by turning to SW-GANs. Another work by \citet{brechet2023critical} focuses on the theoretical properties of linear NNs trained with the Bures-Wasserstein loss (introduced by \citet{bures1969extension}; see also \citep{Bhatia_Bures_metric} for reference on this metric). Finally, the regularity and optimisation properties of the simpler energy $\SW_2^2(\bbgamma_X, \bbgamma_Y)$ have been studied by \citet{discrete_sliced_loss}.

In practice, it has been observed that SGD in such settings always converges (in the loose numerical sense, see \citep{deshpande_generative_sw}, Section 5, or \citep{heitz2021sliced}, Figure 3), yet this property is not known theoretically. The aim of this work is to bridge the gap between theory and practical observation by proving convergence results for SGD on (minibatch) Sliced Wasserstein generative losses of the form $F(\p) = \mathbb{E}_{X\sim \mxn, Y\sim \myn}\SW_2^2(T_u\#\bbgamma_X, \bbgamma_Y)$.
\subsection{Contributions}
\paragraph{Convergence of Interpolated SGD Under Practical Assumptions} Under practically realistic assumptions, we prove in \ref{thm:SGD_interpolated_cv} that piecewise affine interpolations (defined in Equation \ref{eqn:interpolation}) of constant-step SGD schemes on $\p\longmapsto F(\p)$ (formalised in Equation \ref{eqn:SW_SGD}) converge towards the set of sub-gradient flow solutions (see Equation \ref{eqn:S}) as the gradient step decreases. This results signifies that with very small learning rates, SGD trajectories will be close to sub-gradient flows, which themselves converge to critical points of $F$ (omitting serious technicalities). 

The assumptions for this result are practically reasonable: the input measure $\mx$ and the true data measure $\my$ are assumed to be compactly supported. As for the network $(\p, x) \longmapsto T(\p, x)$, we assume that for a fixed datum $x$, $T(\cdot, x)$ is piecewise $\mathcal{C}^2$-smooth and that it is Lipschitz jointly in both variables. 
\redtwo{We require additional assumptions on $T$ which are more costly, but are verified as long as $T$ is a NN composed of typical activations and linear units, with the constraint that the parameters $\p$ and data $x$ stay both stay within a fixed bounded domains. We discuss a class of neural networks that satisfy all of the assumptions of the paper in the Appendix (\ref{sec:suitable_NNs}). Furthermore, this result can be extended to other orders $p\neq 2$ of SW: we present the tools for this generalisation in \ref{sec:other_p}.}

\paragraph{Stronger Convergence Under Stricter Assumptions} In order to obtain a stronger convergence result, we consider a variant of SGD where each iteration receives an additive noise (scaled by the learning rate) which allows for better space exploration, and where each iteration is projected on a ball $B(0, r)$ in order to ensure boundedness. This alternative SGD scheme remains within the realm of practical applications, and we show in \ref{thm:SGD_projected_noised} that long-run limits of such trajectories converge towards a set of generalised critical points of $F$, as the gradient step approaches 0. This result is substantially stronger, and can serve as an explanation of the convergence of practical SGD trajectories, specifically towards a set of critical points which amounts to the stationary points of the energy (barring theoretical technicalities).

Unfortunately, we require additional assumptions in order to obtain this stronger convergence result, the most important of which is that the input data measure $\mx$ and the dataset measure $\my$ are discrete. For the latter, this is always the case in practice, however the former assumption is more problematic, since it is common to envision generative NNs as taking an argument from a continuous space (the input is often Gaussian of Uniform noise), thus a discrete setting is a substantial theoretical drawback. For practical concerns, one may argue that the discrete $\mx$ can have an arbitrary fixed amount of points, and leverage strong sample complexity results to ascertain that the discretisation is not costly if the number of samples is large enough.

\section{Stochastic Gradient Descent with \texorpdfstring{$\SW$}{SW} as Loss}\label{sec:SGD}

Training Sliced-Wasserstein generative models consists in training a neural network
\begin{equation}\label{eqn:T}
	T: \app{\R^{\dimp}\times \R^{\dimx}}{\R^{\dimy}}{(\p, x)}{T_\p(x) := T(\p, x)}
\end{equation}
by minimising the SW minibatch loss $\p \longmapsto\mathbb{E}_{X\sim\mxn, Y\sim\myn}\left[\SW_2^2(T_u\#\bbgamma_X, \bbgamma_Y)\right]$ through Stochastic Gradient Descent (as described in \ref{alg:SGD}). The probability distribution $\mx \in \mathcal{P}_2(\R^{\dimx})$ is the law of the input of the generator $T(\p,\cdot)$. The distribution $\my \in \mathcal{P}_2(\R^{\dimy})$ is the data distribution, which $T$ aims to simulate. Finally, $\bbsigma$ will denote the uniform measure on the unit sphere of $\R^{\dimy}$, denoted by $\SS^{\dimy-1}$. Given a list of points $X = (x_1, \cdots, x_\npoints) \in \R^{\npoints \times \dimx},$ denote the associated discrete uniform measure $\bbgamma_X := \frac{1}{\npoints}\sum_i\bbdelta_{x_i}$. By abuse of notation, we write $T_\p(X) := (T_\p(x_1), \cdots, T_\p(x_\npoints)) \in \R^{\npoints \times \dimy}$. \redtwo{The reader may find a summary of this paper's notations in \ref{tab:notations}.}
\begin{figure}[h]
	\centering
	\begin{minipage}{.95\linewidth}
		\begin{algorithm}[H]
			\SetAlgoLined
			\KwData{Learning rate $\lr > 0$, probability distributions $\mx \in \mathcal{P}_2(\R^{\dimx})$ and $\my \in \mathcal{P}_2(\R^{\dimy})$.}
			\textbf{Initialisation:} Draw $\p^{(0)} \in \R^{\dimp}$\;
			\For{$t \in \llbracket 0, T_{\max} - 1 \rrbracket$}{
				Draw $\theta^{(t+1)} \sim \bbsigma,\;  X^{(t+1)} \sim \mxn\; Y^{(t+1)} \sim \myn$. 
				SGD update: $\p^{(t+1)} = \p^{(t)} - \lr \left[\dr{\p}{}{} \W_2^2(P_{\theta^{(t+1)}}\#T_u\#\bbgamma_{X^{(t+1)}}, P_{\theta^{(t+1)}}\#\bbgamma_{Y^{(t+1)}})\right]_{\p = \p^{(t)}}$
			}
			\caption{Training a NN on the $\SW$ loss with Stochastic Gradient Descent}
			\label{alg:SGD}
		\end{algorithm}
	\end{minipage}
\end{figure}

In the following, we will apply results from \citep{bianchi2022convergence}, and we pave the way to the application of these results by presenting their theoretical framework. Consider a sample loss function $f: \R^{\dimp} \times \Data \longrightarrow \R$ that is locally Lipschitz in the first variable, and $\mdata$ a probability measure on $\Data \subset \R^d$ which is the law of the samples drawn at each SGD iteration. Consider $\varphi: \R^\dimp \times \Data \longrightarrow \R^\dimp$ an \textit{almost-everywhere gradient} of $f$, which is to say that for almost every $(\p, \data) \in \R^{\dimp}\times \Data,\; \varphi(\p, \data) = \partial_\p f(\p, \data)$ (since each $f(\cdot, \data)$ is locally Lipschitz, it is differentiable almost-everywhere by Rademacher's theorem). The complete loss function is the expectation of the sample loss, $F := \p \longrightarrow \int_\Data f(\p, \data)\dd\mdata(\data)$. An SGD trajectory of step $\lr > 0$ for $F$ is a sequence $(\p^{(t)}) \in (\R^\dimp)^\N$ of the form:
$$\p^{(t+1)} = \p^{(t)} - \lr \varphi(\p^{(t)}, \data^{(t+1)}),\quad \left(\p^{(0)}, (\data^{(t)})_{t \in \N}\right) \sim \mpzero \otimes \mdata^{\otimes \N}, $$
where $\mpzero$ is the distribution of the initial position $\p^{(0)}$. Within this framework, we define an SGD scheme described by \ref{alg:SGD}, with $\mdata := \mxn \otimes \myn \otimes \bbsigma$ and the minibatch SW sample loss
\begin{equation}\label{eqn:f}
	f:= \app{\R^{\dimp} \times \R^{\npoints \times \dimx} \times \R^{\npoints \times \dimy} \times \SS^{\dimy-1}}{\R^\dimy}{(\p, X, Y, \theta)}{\W_2^2(P_{\theta}\#T_\p\#\bbgamma_{X}, P_{\theta}\#\bbgamma_{Y})} .
\end{equation}
With this definition for $f$, we have
\begin{equation}\label{eqn:F}
	F(\p) = \mathbb{E}_{(X, Y, \theta) \sim \mdata}\left[f(\p, X, Y, \theta)\right] = \mathbb{E}_{(X, Y) \sim \mxn \otimes \myn}\left[\SW_2^2(T_\p\#\bbgamma_X, \bbgamma_Y)\right],
\end{equation}
thus the population loss compares the "true" data $\my$ with the model's generation $T_\p\#\mx$ using (minibatch) SW. We now wish to define an almost-everywhere gradient of $f$. To this end, notice that one may write $f(\p, X, Y, \theta) = w_\theta(T(\p, X), Y)$, where for $X, Y \in \R^{\npoints \times \dimy}$ and $\theta \in \SS^{\dimy-1},\; w_\theta(X, Y) := \W_2^2(P_\theta\#\bbgamma_X, P_\theta\#\bbgamma_Y)$. The differentiability properties of $w_\theta(\cdot, Y)$ are already known \citep{discrete_sliced_loss, bonneel2015sliced}, in particular one has the following almost-everywhere gradient of $w_\theta(\cdot, Y):$
$$\dr{X}{}{w_\theta}(X, Y) = \left(\cfrac{2}{\npoints}\theta\theta^\top (x_k - y_{\sigma_\theta^{X, Y}(k)})\right)_{k \in \llbracket 1, \npoints \rrbracket} \in \R^{\npoints \times \dimy},$$
where the permutation $\sigma_\theta^{X, Y} \in \mathfrak{S}_\npoints$ is $\sort{Y}{\theta} \circ (\sort{X}{\theta})^{-1}$, with $\sort{Y}{\theta}\in \mathfrak{S}_\npoints$ being a sorting permutation of the list $(\theta^\top  y_1, \cdots, \theta^\top  y_\npoints)$. The sorting permutations are chosen arbitrarily when there is ambiguity. To define an almost-everywhere gradient, we must differentiate $f(\cdot, X, Y, \theta) = \p \longmapsto w_\theta(T(\p, X), Y)$ for which we need regularity assumptions on $T$: this is the goal of \ref{ass:C2_ae}. In the following, $\oll{A}$ denotes the topological closure of a set $A$, $\partial A$ its boundary, and $\bblambda_{\R^\dimp}$ denotes the Lebesgue measure of $\R^\dimp$.

\begin{assumption}\label{ass:C2_ae}
	For every $x \in \R^{\dimx},\;$ there exists a family of disjoint connected open sets $(\U_j(x))_{j \in J(x)}$ such that $\forall j \in J(x),\; T(\cdot, x) \in \mathcal{C}^2(\U_j(x), \R^\dimy)$, $\Reu{j\in J(x)}{}\oll{\U_j(x)} = \R^\dimp$ and $\bblambda_{\R^\dimp}\Big(\Reu{j\in J(x)}{}\partial \U_j(x)\Big) = 0$.
\end{assumption}

Note that for measure-theoretic reasons, the sets $J(x)$ are assumed countable. \redtwo{One may understand this assumption broadly as the neural networks $T$ being piecewise smooth with respect to the parameters $\p$, where the pieces depend on the input data $x$. In practice, \ref{ass:C2_ae} is an assumption on the activation functions of the neural network. For instance, it is of course satisfied in the case of smooth activations, or in the common case of piecewise polynomial activations. We detail suitable neural networks in the Appendix (\ref{sec:suitable_NNs}).}

\ref{ass:C2_ae} implies that given $X, Y, \theta$ fixed, $f(\cdot, X, Y, \theta)$ is differentiable almost-everywhere, and that one may define the following almost-everywhere gradient \ref{eqn:ae_grad}.
\begin{equation}\label{eqn:ae_grad}
	\varphi : \app{\R^{\dimp} \times \R^{\npoints \times \dimx} \times \R^{\npoints \times \dimy} \times \SS^{\dimy-1}}{\R^{\dimp}}{(\p, X, Y, \theta)}{\Sum{k=1}{\npoints} \cfrac{2}{\npoints}\left(\dr{\p}{}{T}(\p, x_k)\right)^\top  \theta \theta^\top (T(\p, x_k) - y_{\sigma_\theta^{T(\p, X), Y}(k)})},
\end{equation}
where for $x \in \R^\dimx,\; \dr{\p}{}{T}(\p, x)\in \R^{\dimy \times \dimp}$ denotes the matrix of the differential of $\p \longmapsto T(\p, x)$, which is defined for almost-every $\p$. Given $\p \in \partial \U_j(x)$ (a point of potential non-differentiability), take instead $0$. (Any choice at such points would still define an a.e. gradient, and will make no difference).

Given a step $\lr > 0$, and an initial position $\p^{(0)} \sim \mpzero$, we may now define formally the following fixed-step SGD scheme for $F$:
\begin{equation}\label{eqn:SW_SGD}
	\begin{split}
		\p^{(t+1)} = \p^{(t)} - \lr \varphi(\p^{(t)}, X^{(t+1)}, Y^{(t+1)}, \theta^{(t+1)}), \\ \left(\p^{(0)}, (X^{(t)})_{t \in \N}\ (Y^{(t)})_{t \in \N}\ (\theta^{(t)})_{t \in \N}\right) \sim \mpzero \otimes \mxN \otimes \myN \otimes \bbsigma^{\otimes \N}. 
	\end{split}
\end{equation}
An important technicality that we must verify in order to apply \citet{bianchi2022convergence}'s results is that $\p \longmapsto f(\p, X, Y, \theta)$ and $F$ are locally Lipschitz. Before proving those claims, we reproduce a useful Property from \citep{discrete_sliced_loss}. In the following, $\|X\|_{\infty, 2}$ denotes $\underset{k \in \llbracket 1, \npoints \rrbracket}{\max}\ \|x_k\|_2$ given $X = (x_1, \cdots, x_\npoints) \in \R^{\npoints \times \dimx}$, and $B_{\mathcal{N}}(x, r)$ for $\mathcal{N}$ a norm on $\R^\dimx$, $x \in \R^\dimx$ and $r>0$ shall denote the open ball of $\R^\dimx$ of centre $x$ and radius $r$ for the norm $\mathcal{N}$ (if $\mathcal{N}$ is omitted, then $B$ is an euclidean ball).

\begin{prop}The $(w_\theta(\cdot, Y))_{\theta \in \SS^{\dimy-1}}$ are uniformly locally Lipschitz \citep{discrete_sliced_loss} Prop. 2.1.\label{prop:w_unif_locLip}\
	
	Let $K_w(r, X, Y) := 2\npoints(r + \|X\|_{\infty, 2} + \|Y\|_{\infty, 2})$, for $X, Y \in \R^{\npoints \times \dimy}$ and $r>0$. Then $w_\theta(\cdot, Y)$ is $K_w(r, X, Y)$-Lipschitz in the neighbourhood $B_{\|\cdot\|_{\infty, 2}}(X, r)$:
	$$\forall Y', Y'' \in B_{\|\cdot\|_{\infty, 2}}(X, r),\; \forall \theta \in \SS^{\dimy-1},\; |w_\theta(Y', Y) - w_\theta(Y'', Y)| \leq K_w(r, X, Y) \|Y'-Y''\|_{\infty, 2}.$$
	
\end{prop}

In order to deduce regularity results on $f$ and $F$ from \ref{prop:w_unif_locLip}, \redtwo{we will make the assumption that $T$ is globally Lipschitz in $(\p, x)$. In practice, this is the case when both parameters are enforced to stay within a fixed bounded domain, for instance by multiplying a typical NN with the indicator of such a set. We present this in detail in the Appendix (\ref{sec:suitable_NNs}).}

\begin{assumption}\label{ass:T_loclip}
	There exists $\lipT>0$ such that 
	$$\forall (\p_1, \p_2, x_1, x_2) \in (\R^{\dimp})^2 \times (\R^{\dimx})^2,\; \|T(\p_1, x_1) - T(\p_2, x_2)\|_2 \leq \lipT\left(\|\p_1-\p_2\|_2 + \|x_1 - x_2\|_2\right).$$
\end{assumption}

\begin{prop}\label{prop:f_loclip} Under \ref{ass:T_loclip}, for $\varepsilon > 0,\; \p_0 \in \R^{\dimp},\; X\in \R^{\npoints \times \dimx},\; Y\in \R^{\npoints \times \dimy}$ and $\theta \in \SS^{\dimy-1}$, let $K_f(\varepsilon, \p_0, X, Y) := 2\lipT\npoints(\varepsilon \lipT + \|T(\p_0, X)\|_{\infty, 2} + \|Y\|_{\infty, 2})$. Then $f(\cdot, X, Y, \theta)$ is $K_f(\varepsilon, \p_0, X, Y)$-Lipschitz in $B(\p_0, \varepsilon)$: 
	$$\forall \p, \p' \in B(\p_0, \varepsilon),\; |f(\p, X, Y, \theta) - f(\p', X, Y, \theta)| \leq K_f(\varepsilon, \p_0, X, Y)\|u-u'\|_2.$$
	
\end{prop}

\begin{proof}
	Let $\varepsilon > 0,\; \p_0 \in \R^{\dimp},\; X\in \R^{\npoints \times \dimx},\; Y\in \R^{\npoints \times \dimy}$ and $\theta \in \SS^{\dimy-1}$. Let $\p, \p' \in B(\p_0, \varepsilon)$. Using \ref{ass:T_loclip}, we have $T(\p, X), T(\p', X) \in B_{\|\cdot\|_{\infty, 2}}(T(\p_0, X), r)$, with $r := \varepsilon\lipT$.	
	
	\red{Denoting $\lipT := \lipT_{\oll{B}(u_0, \varepsilon), \oll{B}(0_{\R^\dimx}, \|X\|_{\infty, 2})}$, we apply successively \ref{prop:w_unif_locLip} (first inequality), then \ref{ass:T_loclip} (second inequality):}
	\begin{align*}|f(\p, X, Y, \theta) - f(\p', X, Y, \theta)| &= |w_\theta(T(\p, X), Y) - w_\theta(T(\p', X), Y)|\\
		&\leq K_w(r, T(\p_0, X), Y) \|T(\p, X) - T(\p', X)\|_{\infty, 2} \\
		&\leq 2\npoints(\varepsilon \lipT + \|T(\p_0, X)\|_{\infty, 2} + \|Y\|_{\infty, 2}) \lipT\|u-u'\|_2.
	\end{align*}
	\vspace{-10pt}
\end{proof}

\ref{prop:f_loclip} shows that $f$ is locally Lipschitz in $\p$. We now assume some conditions on the measures $\mx$ and $\my$ in order to prove that $F$ is also locally Lipschitz. \redtwo{Specifically, we require that the data measures $\mx$ and $\my$ be supported on bounded domains, which imposes little restriction in practice.}

\begin{assumption}\label{ass:mx_my}
	$\mx$ and $\my$ are Radon probability measures on $\R^\dimx$ and $\R^\dimy$ respectively, supported by the compacts $\X$ and $\Y$ respectively. Denote $R_x := \underset{x \in \X}{\sup}\ \|x\|_2$ and $R_y := \underset{y \in \Y}{\sup}\ \|y\|_2$.
\end{assumption}

\begin{prop}\label{prop:F_loclip}
	Assume \ref{ass:T_loclip} and \ref{ass:mx_my}. For $\varepsilon > 0,\; \p_0 \in \R^{\dimp},$ let $C_1(\p_0) := \Int{\Xn}{}\|T(\p_0, X)\|_{\infty, 2} \dd\mxn(X)$ and $C_2 := \Int{\Yn}{}\|Y\|_{\infty, 2}\dd \myn(Y)$.
	
	Let $K_F(\varepsilon, \p_0) := 2\lipT\npoints(\varepsilon \lipT + C_1(\p_0) + C_2)$. We have $\forall \p, \p' \in B(\p_0, \varepsilon),\; |F(\p) - F(\p')| \leq K_F(\varepsilon, \p_0)\|\p-\p'\|_2$.
\end{prop}

\begin{proof}
	Let $\varepsilon > 0,\; \p_0 \in \R^{\dimp}$ and $\p, \p' \in B(\p_0, \varepsilon)$. We have	
	\begin{align*} |F(\p) - F(\p')| &\leq \Int{\Xn \times \Yn \times \SS^{\dimy-1}}{}|f(\p, X, Y, \theta) - f(\p', X, Y, \theta)|\dd \mxn(X) \dd \myn(Y) \dd\bbsigma(\theta) \\
		&\leq \Int{\Xn \times \Yn}{}K_f(\varepsilon, \p_0, X, Y)\|\p-\p'\|_2\dd \mxn(X) \dd \myn(Y)\\
		&\leq \Int{\Xn \times \Yn}{}2\lipT\npoints(\varepsilon \lipT + \|T(\p_0, X)\|_{\infty, 2} + \|Y\|_{\infty, 2})\|\p-\p'\|_2\dd \mxn(X) \dd \myn(Y).
	\end{align*}
	\red{Now by \ref{ass:T_loclip}, $X \longmapsto \|T(\p_0, X)\|_{\infty, 2}$ is continuous on the compact $\Xn$, thus upper-bounded by a certain $M(u_0) > 0$. We can define $C_1(\p_0) := \Int{\Xn}{}\|T(\p_0, X)\|_{\infty, 2} \dd\mxn(X)$, which verifies $C_1(\p_0) \leq M(u_0) \mx(\X)^n$. Since $\X$ is compact and $\mx$ is a Radon probability measure by \ref{ass:mx_my}, $\mx(\X)$ is well-defined and finite, thus $C_1(u_0)$ is finite. Likewise, let $C_2 := \Int{\Yn}{}\|Y\|_{\infty, 2}\dd \myn(Y) <+\infty$.}
	
	Finally, $|F(\p) - F(\p')| \leq 2\lipT\npoints(\varepsilon \lipT + C_1(\p_0) + C_2)\|\p-\p'\|_2$.
\end{proof}

Having shown that our losses are locally Lipschitz, we can now turn to convergence results. These conclusions are placed in the context of non-smooth and non-convex optimisation, thus will be tied to the Clarke sub-differential of $F$, which we denote $\partial_C F$. The set of Clarke sub-gradients at a point $\p$ is the convex hull of the limits of gradients of $F$:
\begin{equation}\label{eqn:clarke}
	\partial_C F(\p) := \mathrm{conv}\left\{v \in \R^\dimp:\; \exists (\p^{(t)}) \in (\mathcal{D}_F)^\N: \p^{(t)} \xrightarrow[t \longrightarrow +\infty]{} \p\ \mathrm{and} \; \nabla F(\p^{(t)}) \xrightarrow[t \longrightarrow +\infty]{} v\right\},
\end{equation}
where $\mathcal{D}_F$ is the set of differentiability of $F$. At points $\p$ where $F$ is differentiable, $\partial_CF(\p) = \{\nabla F(\p)\}$, and if $F$ is convex in a neighbourhood of $\p$, then the Clarke differential at $\p$ is the set of its convex sub-gradients. \redtwo{The interested reader may turn to \ref{sec:nonsmooth} for further context on non-smooth and non-convex optimisation.}

\section{Convergence of Interpolated SGD Trajectories on \texorpdfstring{$F$}{F}}\label{sec:interpolated_SGD}

In general, the idea behind SGD is a discretisation of the gradient flow equation $\dot{u}(s) = -\nabla F(u(s))$. In our non-smooth setting, the underlying continuous-time problem is instead the Clarke differential inclusion $\dot{u}(s) \in -\partial_C F(u(s))$. Our objective is to show that in a certain sense, the SGD trajectories approach the set of solutions of this inclusion problem, as the step size decreases. We consider solutions that are absolutely continuous (we will write $\p(\cdot) \in \mathcal{C}_{\mathrm{abs}}(\R_+, \R^\dimp)$) and start within $\mathcal{K} \subset \R^\dimp$, a fixed compact set. We can now define the solution set formally as
\begin{equation}\label{eqn:S}
	S_{-\partial_C F}(\mathcal{K}) := \left\lbrace \p \in \mathcal{C}_{\mathrm{abs}}(\R_+, \R^\dimp)\ |\  \ull{\forall} s\in \R_+,\; \dot{\p}(s) \in -\partial_C F(\p(s));\; \p(0) \in \mathcal{K} \right\rbrace,
\end{equation}
where we write $\ull{\forall}$ for "almost every". In order to compare the discrete SGD trajectories to this set of continuous-time trajectories, we interpolate the discrete points in an affine manner: Equation \ref{eqn:interpolation} defines the \textit{piecewise-affine interpolated SGD trajectory} associated to a discrete SGD trajectory $(\p_\lr^{(t)})_{t \in \N}$ of learning rate $\lr$.
\begin{equation}\label{eqn:interpolation}
	\p_\lr(s) = \p^{(t)}_\lr + \left(\frac{s}{\lr} - t\right)(\p_\lr^{(t+1)} - \p^{(t)}_\lr),\quad \forall s \in [t\lr, (t+1)\lr[, \quad \forall t \in \N.
\end{equation}
In order to compare our interpolated trajectories with the solutions, we consider the metric of uniform convergence on all segments
\begin{equation}\label{eqn:d_c}
	d_c(\p, \p') := \Sum{k \in \N^*}{}\cfrac{1}{2^k}\min\left(1, \underset{s \in [0, k]}{\max}\|\p(s) - \p'(s)\|_{2}\right).
\end{equation}
In order to prove a convergence result on the interpolated trajectories, we will leverage the work of \citet{bianchi2022convergence} which hinges on three conditions on the loss $F$ that we reproduce and verify successively. \redtwo{Firstly, \ref{cond:A1} assumes mild regularity on the sample loss function $f$.}

\begin{condition}\label{cond:A1}\
	\begin{itemize}		
		\item[i)] There exists $\kappa: \R^\dimp \times \Data \longrightarrow \R_+$ measurable such that each $\kappa(\p, \cdot)$ is $\mdata$-integrable, and:		
		$$\exists \varepsilon > 0,\; \forall \p, \p' \in B(\p_0, \varepsilon),\; \forall \data \in \Data,\; |f(\p, \data) - f(\p', \data)|\leq \kappa(\p_0, \data)\|\p-\p'\|_2. $$		
		\item[ii)] There exists $\p \in \R^\dimp$ such that $f(\p, \cdot)$ is $\mdata$-integrable.
	\end{itemize}
\end{condition}

Our regularity result on $f$ \ref{prop:f_loclip} allows us to verify \ref{cond:A1}, by letting $\varepsilon := 1$ and $\kappa(\p_0, \data) := K_f(1, \p_0, X, Y)$. \ref{cond:A1} ii) is immediate since for \textit{all} $\p \in \R^\dimp,\; (X, Y, \theta) \longmapsto w_\theta(T(\p, X), Y)$ is continuous in each variable separately, thanks to the regularity of $T$ provided by \ref{ass:T_loclip}, and to the regularities of $w$. This continuity implies that all $f(\p, \cdot)$ are $\mdata$-integrable, since $\mdata = \mxn \otimes \myn \otimes \bbsigma$ is a compactly supported probability measure under \ref{ass:mx_my}. \redtwo{Secondly, \ref{cond:A2} concerns the local Lipschitz constant $\kappa$ introduced in \ref{cond:A1}: it is assumed to increase slowly with respect to the network parameters $\p$.}

\begin{condition}\label{cond:A2} The function $\kappa$ of \ref{cond:A1} verifies:	
	\begin{itemize}
		\item[i)] There exists $c \geq 0$ such that $\forall \p\in \R^\dimp,\; \Int{\Data}{}\kappa(\p, \data)\dd\mdata(\data) \leq c(1+\|\p\|_2)$.		
		\item[ii)] For every compact $\mathcal{K} \subset \R^\dimp,\; \underset{\p \in \mathcal{K}}{\sup}\ \Int{\Data}{}\kappa(\p, \data)^2\dd\mdata(\data) <+\infty$.
	\end{itemize}	
\end{condition}

\ref{cond:A2}.ii) is verified by $\kappa$ given its regularity. However, \ref{cond:A2}.i) requires that $T(\p, x)$ increase slowly as $\|\p\|_2$ increases, which is more costly.

\begin{assumption}\label{ass:T_slow_increase}
	There exists an $\mx$-integrable function $g: \R^\dimx \longrightarrow\R_+$ such that $\forall \p \in \R^\dimp,\; \forall x \in \R^\dimx,\; \|T(\p, x)\|_2 \leq g(x)(1 + \|\p\|_2)$.
\end{assumption}

\ref{ass:T_slow_increase} is satisfied in particular as soon as $T(\cdot, x)$ is bounded (which is the case for a neural network with bounded activation functions), or if $T$ is of the form $T(\p, x) = \widetilde{T}(\p, x) \mathbbold{1}_{B(0, R)}(\p)$, i.e. limiting the network parameters $\p$ to be bounded. This second case does not yield substantial restrictions in practice \redtwo{(see \ref{sec:suitable_NNs} for a class of NNs that satisfy all of the assumptions)}, yet vastly simplifies theory.  Under \ref{ass:T_slow_increase}, we have for any $\p \in \R^\dimp,$ with $\kappa(\p, \data) = K_f(1, \p, X, Y)$ from \ref{prop:f_loclip} and $C_2$ from \ref{prop:F_loclip},
\begin{align*}\Int{\Xn \times \Yn \times \SS^{\dimy-1}}{}K_f(1, \p, X, Y)\dd\mxn(X) \dd\myn(Y) \dd\bbsigma(\theta) &\leq 4\lipT\npoints\left(\varepsilon \lipT + (1+\|\p\|_2)\Int{\Xn}{}\underset{k \in \llbracket 1, \npoints \rrbracket}{\max}\ g(x_k) \dd\mxn(X) + C_2\right)\\
	&\leq c(1 + \|\p\|_2).
\end{align*}
As a consequence, \ref{cond:A2} holds under our assumptions. We now consider the Markov kernel associated to the SGD schemes:
$$P_\lr : \app{\R^\dimp \times \mathcal{B}(\R^\dimp)}{[0, 1]}{\p, B}{\Int{\Data}{}\mathbbold{1}_B(\p - \lr \varphi(\p, \data))\dd \mdata(\data)}.$$
\redtwo{Given $\p \in \R^{\dimp},\; P_\lr(\p, \cdot)$ is a probability measure on $\R^{\dimp}$ which dictates the law of the positions of the next SGD iteration $\p^{(t+1)}$, conditionally to $\p^{(t)} = \p$.} With $\bblambda_{\R^\dimp}$ denoting the Lebesgue measure on $\R^\dimp$, let $\Gamma := \left\lbrace \lr \in\ ]0, +\infty[\ |\ \forall \mp \ll \bblambda_{\R^\dimp},\ \mp P_\lr \ll \bblambda_{\R^\dimp}\right\rbrace$. $\Gamma$ is the set of learning rates $\alpha$ for which the kernel $P_\alpha$ maps any absolutely continuous probability measure $\mp$ to another such measure. We will verify the following condition, \redtwo{which can be interpreted as the SGD trajectories continuing to explore the entire space for a small enough learning rate $\lr$}:
\begin{condition}\label{cond:A3} The closure of $\Gamma$ contains 0.
\end{condition}

In order to satisfy \ref{cond:A3}, we require an additional regularity condition on the neural network $T$ which we formulate in \ref{ass:d2T2_and_dT_bounded}.

\begin{assumption}\label{ass:d2T2_and_dT_bounded}
	There exists a constant $\MddT>0$, such that  (with the notations of \ref{ass:C2_ae} and \ref{ass:mx_my}) $\forall x \in \X,\; \forall j \in J(x),\; \forall \p \in \U_j(x),\; \forall (i_1, i_2, i_3, i_4) \in \llbracket 1, \dimp\rrbracket^2\times \llbracket 1, \dimy \rrbracket^2,$	
	$$ \left|\cfrac{\partial^2}{\partial \p_{i_1} \partial \p_{i_2}} \Big([T(\p, x)]_{i_3} [T(\p, x)]_{i_4}\Big)\right| \leq \MddT,\; \mathrm{and}\;\left\|\cfrac{\partial^2 T}{\partial \p_{i_1}\partial\p_{i_2}}(\p, x)\right\|_2 \leq \MddT.$$
\end{assumption}

The upper bounds in \ref{ass:d2T2_and_dT_bounded} bear strong consequences on the behaviour of $T$ for $\|u\|_2 \gg 1$, and are only practical for networks of the form $T(\p, x) = \widetilde{T}(\p, x) \mathbbold{1}_{B(0,R)}(\p, x)$, similarly to \ref{ass:T_slow_increase}. \redtwo{We detail the technicalities of verifying this assumption along with the others in the Appendix (\ref{sec:suitable_NNs}).}

\begin{prop}\label{prop:SW_Gamma} Under \ref{ass:C2_ae}, \ref{ass:mx_my} and \ref{ass:d2T2_and_dT_bounded}, for the SGD trajectories \ref{eqn:SW_SGD}, $\Gamma$ contains $]0, \lr_0[$, where $\lr_0 := \left((\dimy^2 + 2R_y )\dimp\MddT\right)^{-1}$.
\end{prop}

We postpone the proof to \ref{sec:proof_SW_gamma}. Now that we have verified \ref{cond:A1}, \ref{cond:A2} and \ref{cond:A3}, we can apply \citep{bianchi2022convergence}, Theorem 2 to $F$, showing a convergence result on interpolated SGD trajectories.

\begin{theorem}\label{thm:SGD_interpolated_cv} Consider a neural network $T$ and measures $\mx$, $\my$ satisfying \ref{ass:C2_ae}, \ref{ass:T_loclip}, \ref{ass:mx_my}, \ref{ass:T_slow_increase} and \ref{ass:d2T2_and_dT_bounded}. Let $\lr_1 < \lr_0$ (see \ref{prop:SW_Gamma}).
	
	Let $(\p_\lr^{(t)}), \lr \in ]0, \lr_1], t \in \N$ a collection of SGD trajectories associated to~\ref{eqn:SW_SGD}. Consider $(\p_\lr)$ their associated interpolations. For any compact $\mathcal{K} \subset \R^\dimp$ and any $\eta > 0$, we have:	
	\begin{equation}
		\underset{\substack{\lr \longrightarrow 0 \\ \lr \in\ ]0, \lr_1]}}{\lim}\ \mpzero \otimes \mxN \otimes \myN \otimes \bbsigma^{\otimes \N} \left( d_c(\p_\lr, S_{-\partial_C F}(\mathcal{K})) > \eta\right) = 0.
	\end{equation}	
\end{theorem}
\vspace{-10pt}
The distance $d_c$ is defined in \ref{eqn:d_c}. As the learning rate decreases, the interpolated trajectories approach the trajectory set $S_{-\partial_C F}$, which is essentially a solution of the \textit{gradient flow equation}  $\dot{u}(s) = -\nabla F(u(s))$ (ignoring the set of non-differentiability, which is $\bblambda_{\R^\dimp}$-null). To get a tangible idea of the concepts at play, if $F$ was $\mathcal{C}^2$ and had a finite amount of critical points, then one would have the convergence of a solution $u(s)$ to a critical point of $F$, as $s \longrightarrow+\infty$. These results have implicit consequences on the value of the parameters  at the "end" of training for low learning rates, which is why we will consider a variant of SGD for which we can say more precise results on the convergence of the parameters.

\section{Convergence of Noised Projected SGD Schemes on \texorpdfstring{$F$}{F}}\label{sec:noised_proj_sgd}

In practice, it is seldom desirable for the parameters of a neural network to reach extremely large values during training. Weight clipping is a common (although contentious) method of enforcing that $T(\p, \cdot)$ stay Lipschitz, which is desirable for theoretical reasons. For instance the 1-Wasserstein duality in Wasserstein GANs \citep{pmlr-v70-arjovsky17a} requires Lipschitz networks, and similarly, Sliced-Wasserstein GANs \citep{deshpande_generative_sw} use weight clipping and enforce their networks to be Lipschitz.

Given a radius $r > 0$, we consider SGD schemes that are restricted to $\p \in \oll{B}(0, r) =: B_r$, by performing \textit{projected} SGD. At each step $t$, we also add a noise $\noise \varepsilon^{(t+1)}$, where $\varepsilon^{(t+1)}$ is an additive noise of law $\mnoise \ll \bblambda_{\R^\p}$, which is often taken as standard Gaussian in practice. These additions yield the following SGD scheme:
\begin{equation}\label{eqn:SW_SGD_projected_noised}
	\begin{split}
		\p^{(t+1)} = \pi_r\left(\p^{(t)} - \lr \varphi(\p^{(t)}, X^{(t+1)}, Y^{(t+1)}, \theta^{(t+1)}) + \lr\noise \varepsilon^{(t+1)}\right), \\ \left(\p^{(0)}, (X^{(t)})_{t \in \N}\ (Y^{(t)})_{t \in \N},\ (\theta^{(t)})_{t \in \N},\ (\varepsilon^{(t)})_{t \in \N}\right) \sim \mpzero \otimes \mxN \otimes \myN \otimes \bbsigma^{\otimes \N} \otimes \mnoise^{\otimes\N},
	\end{split}
\end{equation}
where $\pi_r: \R^\p \longrightarrow B_r$ denotes the orthogonal projection on the ball $B_r := \oll{B}(0, r)$. Thanks to \ref{cond:A1}, \ref{cond:A2} and the additional noise, we can verify the assumptions for \citep{bianchi2022convergence} Theorem 4, yielding the same result as \ref{thm:SGD_interpolated_cv} for the noised projected scheme \ref{eqn:SW_SGD_projected_noised}. In fact, under additional assumptions, we shall prove a stronger mode of convergence for the aforementioned trajectories. The natural context in which to perform gradient descent is on functions that admit a chain rule, which is formalised in the case of almost-everywhere differentiability by the notion of \textit{path differentiability}, as studied thoroughly in  \citep{bolte2021conservative}. \redtwo{We also provide a brief presentation in the Appendix (\ref{sec:conservative_fields}).} %

\begin{condition}\label{cond:A5}$F$ is path differentiable, which is to say that for any $\p \in \mathcal{C}_{\mathrm{abs}}(\R_+, \R^\dimp)$, for almost all $s > 0,\; \forall v \in \partial_CF(\p(s)),\; v^\top  \dot \p(s) = (F \circ \p)'(s)$.
\end{condition} 

\begin{remark} \red{There are alternate equivalent formulations for \ref{cond:A5}. Indeed, as presented in further detail in \ref{sec:conservative_fields}, $F$ is path differentiable if and only if $\partial_C F$ is a conservative field for $F$ if and only if $F$ has a chain rule for $\partial_C$ (the latter is the formulation chosen above in \ref{cond:A5}).}
\end{remark}

In order to satisfy \ref{cond:A5}, we need to make the assumption that the NN input measure $\mx$ and the data measure $\my$ are discrete measures, which is the case for $\my$ in the case of generative neural networks, but is less realistic for $\mx$ in practice. We define $\Sigma_n$ the $n$-simplex: its elements are the $a \in \R^n$ s.t. $\forall i \in \llbracket 1, n \rrbracket,\; a_i \geq 0$ and $\sum_i a_i =1$.

\begin{assumption}\label{ass:mx_my_discrete}
	One may write $\mx = \Sum{k=1}{\npoints_x}a_k \bbdelta_{x_k}$ and $\my = \Sum{k=1}{\npoints_y}b_k \bbdelta_{y_k}$, with the coefficient vectors $a\in \Sigma_{\npoints_x}, b\in \Sigma_{\npoints_y}$, $\X = \lbrace x_1, \cdots, x_{\npoints_x}\rbrace \subset \R^{\dimx}$ and $\Y = \lbrace y_1, \cdots, y_{\npoints_y}\rbrace \subset \R^{\dimy}$.
\end{assumption}  

There is little practical reason to consider non-uniform measures, however the generalisation to any discrete measure makes no theoretical difference. Note that \ref{ass:mx_my} is clearly implied by \ref{ass:mx_my_discrete}. 

In order to show that $F$ is path differentiable, we require the natural assumption that each $T(\cdot, x)$ be path differentiable. Since $T(\cdot, x)$ is a vector-valued function, we need to extend the notion of path-differentiability. Thankfully, \citet{bolte2021conservative} define \textit{conservative mappings} for vector-valued locally Lipschitz functions (Definition 4), which allows us to define naturally path differentiability of a vector-valued function as the path-differentiability of all of its coordinate functions. \red{See \ref{sec:conservative_mappings} for a detailed presentation.}

\begin{assumption}\label{ass:T_path_diff}
	For any $x \in \R^\dimx,\; T(\cdot, x)$ is path differentiable.
\end{assumption}

\red{\ref{ass:T_path_diff} holds as soon as each the neural network has the typical structure of compositions of linear units and typical activations, as was proved by \citet{davis2020stochastic}, Corollary 5.11 and \citet{bolte2021conservative}, Section 6.2.} \redtwo{We provide a more specific class of NNs that are path differentiable and satisfy all our other assumptions in \ref{sec:suitable_NNs}}.

\begin{prop}\label{prop:F_path_diff}
	Under \ref{ass:T_loclip}, \ref{ass:mx_my_discrete} and \ref{ass:T_path_diff}, $F$ is path differentiable.
\end{prop}

\begin{proof}
	We shall use repeatedly the property that the composition of path differentiable functions remains path differentiable, which is proved in \citep{bolte2021conservative}, Lemma 6. 
	
	Let $\SWY: \app{\R^{\npoints \times \dimy} \times \R^{\npoints \times \dimy}}{\R_+}{Y, Y'}{\SW_2^2(\bbgamma_Y, \bbgamma_{Y'})}$. By \citep{discrete_sliced_loss}, Proposition 2.4.3, each $\SWY(\cdot, Y)$ is semi-concave and thus is path differentiable (by \citep{discrete_sliced_loss}, Proposition 4.3.3).
	
	Thanks to \ref{ass:mx_my_discrete}, $\mxn$ and $\myn$ are discrete measures on $\R^{\npoints \times \dimx}$ and $\R^{\npoints \times \dimy}$ respectively, allowing one to write $\mxn = \sum_k a_k\bbdelta_{X_k}$ and $\myn = \sum_l b_l\bbdelta_{Y_l}$. Then $F = \p \longmapsto \sum_{k,l}a_kb_l\SWY(T(\p, X_k), Y_l)$ is path differentiable as a sum (\citep{bolte2021conservative}, Corollary 4) of compositions (\citep{bolte2021conservative}, Lemma 6) of path differentiable functions.
\end{proof}

We have now satisfied all the assumptions to apply \citep{bianchi2022convergence}, Theorem 6, showing that trajectories of \ref{eqn:SW_SGD_projected_noised} converge towards to a set of generalised critical points\footnote{\red{Typically referred to as the set of \textit{Karush-Kahn-Tucker} points of the differential inclusion $\dot u(s) \in -\partial_C F(u(s)) - \mathcal{N}_r(u(s))$.}} $\mathcal{C}_r$ defined as
\begin{equation}\label{eqn:KKT_points}
	\mathcal{C}_r := \left\lbrace \p \in \R^\dimp\ |\ 0 \in -\partial_CF(\p) - \mathcal{N}_r(\p) \right\rbrace, \quad \mathcal{N}_r(\p) = \left\lbrace\begin{array}{c}
		\lbrace 0 \rbrace\ \mathrm{if}\ \|\p\|_2 < r \\
		\lbrace s \p\ |\ s \geq 0 \rbrace\ \mathrm{if}\ \|\p\|_2 = r \\
		\varnothing\ \mathrm{if}\ \|\p\|_2 > r
	\end{array} \right. ,
\end{equation}
where $\mathcal{N}_r(\p)$ refers to the \textit{normal cone} of the ball $\oll{B}(0, r)$ at $x$. The term $\mathcal{N}_r(\p)$ in \ref{eqn:KKT_points} only makes a difference in the pathological case $\|\p\|_2 = r$, which never happens in practice since the idea behind projecting is to do so on a very large ball, in order to avoid gradient explosion, to limit the Lipschitz constant and to satisfy theoretical assumptions. Omitting the $\mathcal{N}_r(\p)$ term, and denoting $\mathcal{D}$ the points where $F$ is differentiable, \ref{eqn:KKT_points} simplifies to $\mathcal{C}_r \cap \mathcal{D} = \lbrace\p \in \mathcal{D}\ |\ \nabla F(\p) = 0\rbrace$, i.e. the critical points of $F$ for the usual differential. Like in \ref{thm:SGD_interpolated_cv}, we let $\lr_1 < \lr_0$, where $\lr_0$ is defined in \ref{prop:SW_Gamma}. We have met the conditions to apply \cite{bianchi2022convergence}, Theorem 6, showing a long-run convergence results on the SGD trajectories \ref{eqn:SW_SGD_projected_noised}.

\begin{theorem}\label{thm:SGD_projected_noised}
	
	Consider a neural network $T$ and measures $\mx$, $\my$ satisfying \ref{ass:C2_ae}, \ref{ass:T_loclip}, \ref{ass:T_slow_increase}, \ref{ass:d2T2_and_dT_bounded}, \ref{ass:mx_my_discrete} and \ref{ass:T_path_diff}. Let $(\p_{\lr}^{(t)})_{t \in \N}$ be SGD trajectories defined by \ref{eqn:SW_SGD_projected_noised} for $r > 0$ and $\lr \in ]0, \lr_1]$. One has	
	$$\forall \eta > 0,\; \underset{t \longrightarrow +\infty}{\oll{\lim}}\ \mpzero \otimes \mxN\otimes \myN \otimes \bbsigma^{\otimes \N} \otimes \mnoise^{\otimes \N}\left(d(\p_{\lr}^{(t)}, \mathcal{C}_r) > \eta\right) \xrightarrow[\substack{\lr \longrightarrow 0\\ \lr \in \left] 0, \lr_1\right]}]{} 0.$$
\end{theorem}
\vspace{-10pt}
The distance $d$ above is the usual euclidean distance. \ref{thm:SGD_projected_noised} shows essentially that as the learning rate approaches 0, the long-run limits of the SGD trajectories approach the set of $\mathcal{C}_r$ in probability. Omitting the points of non-differentiability and the pathological case $\|u\|_2=r$, the general idea is that $u_\alpha^{(\infty)} \xrightarrow[\lr \longrightarrow 0]{} \lbrace u\ :\ \nabla F(u)=0 \rbrace$, which is the convergence that would be achieved by the gradient flow of $F$, in the simpler case of $\mathcal{C}^2$ smoothness.

\section{Conclusion and Outlook}

Under reasonable assumptions, we have shown that SGD trajectories of parameters of generative NNs with a minibatch SW loss converge towards the desired sub-gradient flow solutions, implying in a weak sense the convergence of said trajectories. Under stronger assumptions, we have shown that trajectories of a mildly modified SGD scheme converge towards a set of generalised critical points of the loss, which provides a missing convergence result for such optimisation problems.

The core limitation of this theoretical work is the assumption that the input data measure $\mx$ is discrete (\ref{ass:mx_my_discrete}), which we required in order to prove that the loss $F$ is path differentiable. In order to generalise to a non-discrete measure, one would need to apply or show a result on the stability of path differentiability through integration: in our case, we want to show that $\int_{\Xn} \SWY(T(\p, X), Y)\dd\mxn(X)$ is path differentiable, knowing that $\p \longmapsto \SWY(T(\p, X), Y)$ is path differentiable by composition (see the proof of \ref{prop:F_path_diff} for the justification). Unfortunately, in general if each $g(\cdot, x)$ is path differentiable, it is not always the case that $\int g(\cdot, x)\dd x$ is path differentiable (at the very least, there is no theorem stating this, even in the simpler case of another sub-class of path differentiable functions, see \citep{bianchi2022convergence}, Section 6.1). However, there is such a theorem (specifically \citep{clarke1990optimization}, Theorem 2.7.2 with Remark 2.3.5) for \textit{Clarke regular} functions \red{(see \ref{sec:Clarke_regular} for a presentation of this regularity class)}, sadly the composition of Clarke regular functions is not always Clarke regular, it is only known to be the case in excessively restrictive cases (see \citep{clarke1990optimization}, Theorems 2.3.9 and 2.3.10). \redtwo{Similarly to the continuous case, the simpler generalisation in which $\mx$ has a countable support adds substantial difficulty, since all of the typical tools (path differentiability itself, Clarke regularity or even definability (see \citep{bolte2021conservative} Section 4.1 for a first introduction) do not have readily applicable results for infinite operations, to our knowledge. As a result, we leave the generalisation to a non-discrete input measure $\mx$ for future work.}

\redtwo{Our studies focus on the 2-SW distance, but our results from \ref{sec:interpolated_SGD} can be extended to $p\in [1, +\infty[$, as presented in the appendix (\ref{sec:other_p}). However, as also discussed in the Appendix, the generalisation of \ref{sec:noised_proj_sgd} is still an open problem, since it has not yet be proven that $X \longmapsto \SW_p^p(\bbgamma_X, \bbgamma_Y)$ is path differentiable for $p\neq 2$.}

\redtwo{This paper studies the use of the \textit{average} SW distance as a loss, and an extension to related distances would be worth considering. The average SW distance aggregates the projected distances through an expectation, while the closely-related \textit{max}-Sliced Wasserstein distance introduced by \citet{deshpande2019max} aggregates the projections via a maximisation on the axis $\theta \in \SS^{d-1}$. The training paradigm presented in \citep{deshpande2019max} differs strongly from our formalism since it applies to GANs, however one could consider an extension of our formalism in which the optimal projection $\theta$ becomes a learned parameter of the neural network. A related extension is the Subspace-Robust Wasserstein distance \citep{paty2019subspace}, which can take the following formulation
$$\mathcal{S}^2_k(\mx, \my) =\underset{\substack{0 \preceq \Omega \preceq I_d \\ \mathrm{trace}(\Omega) = k}}{\max}\ \W_2^2(\Omega^{1/2}\#\mx, \Omega^{1/2}\#\my),$$
for which one could consider a similar extension where the positive semi-definite $\Omega$ becomes a learned parameter of $T$.} 

Another avenue for future study would be to tie the flow approximation result from \ref{thm:SGD_interpolated_cv} to Sliced Wasserstein Flows \citep{liutkus19a_SWflow_generation, bonet2022efficient}. The difficulty in seeing the differential inclusion \ref{eqn:S} as a flow of $F$ lies in the non-differentiable nature of the functions at play, as well as the presence of the composition between SW and the neural network $T$, which bodes poorly with Clarke sub-differentials. 
\subsection*{Acknowledgements}

We thank Julie Delon for proof-reading and general feedback, as well as Rémi Flamary and Alain Durmus for fruitful discussions.

\bibliography{ecl}
\bibliographystyle{tmlr}

\appendix
\redtwo{\section{Table of Notations}\label{sec:notations}}

\begin{table}[H]
	\centering
	\caption{List of Notations}
	\begin{tabular}{|c|c|}
		\hline
		\textbf{Symbol} & \textbf{Explanation} \\
		\hline
		$\bbgamma_X$ & Given $X = (x_1, \cdots, x_n) \in \R^{n \times d},\; \bbgamma_X = \frac{1}{n}\sum_i\bbdelta_{x_i}$\\
		\hline
		$X$ & $(x_1, \cdots, x_n) \in \R^{\npoints \times \dimx}$ an input data sample of law $\mxn$ \\
		\hline
		$\mx$ & input data probability measure on $\R^\dimx$, supported on $\X$ \\
		\hline
		$Y$ & $(y_1, \cdots, y_n) \in \R^{\npoints \times \dimy}$ a target data sample of law $\myn$ \\
		\hline
		$\my$ & target data probability measure on $\R^\dimy$, supported on $\Y$  \\
		\hline
		$\theta$ & direction in $\SS^{\dimy -1}$\\
		\hline
		$\bbsigma$ & uniform measure on $\SS^{\dimy-1}$\\
		\hline
		$\data := (X, Y, \theta)$ & sample in $X, Y$ and $\theta$\\
		\hline
		$\mdata := \mxn \otimes \myn \otimes \bbsigma$ & probability measure for the samples $\data$, supported on $\Data := \Xn \times \Yn \times \SS^{\dimy - 1}$ \\
		\hline 
		$\p$ & neural network parameters in $\R^\dimp$ \\
		\hline
		$T(\p, X)$ & neural network function defined in \ref{eqn:T} \\
		\hline 
		$f(\p, X, Y, \theta)$ & sample loss function defined in \ref{eqn:f} \\
		\hline
		$F(\p)$ & population loss function defined in \ref{eqn:F} \\
		\hline
		$w_\theta(Y, Y')$ & discrete and projected 2-Wasserstein distance $\W_2^2(P_\theta\#\bbgamma_Y, P_\theta\#\bbgamma_{Y'})$ \\
		\hline
		$\varphi(\p, X, Y, \theta)$ & almost-everywhere gradient of $f(\cdot, X, Y, \theta)$ defined in \ref{eqn:ae_grad}\\
		\hline
		$K_w, K_f, K_F$ & local Lipschitz constants of $w, f, F$ respectively (see Propositions 1, 2, 3)\\
		\hline
		$\lr; \noise$ & SGD learning rate; noise level \\
		\hline
		$\bblambda_{\R^d}; \rho \ll \bblambda_{\R^d}$ & Lebesgue measure on $\R^d$; a measure $\rho$ absolutely continuous w.r.t. $\bblambda_{\R^d}$\\
		\hline
		$\partial_C$ & Clarke differential, defined in \ref{eqn:clarke}\\
		\hline
		$\mpzero$ & probability measure of SGD initialisation $\p^{(0)}$\\
		\hline
		$\varepsilon^{(t)}$ & additive noise in $\R^{\dimp}$ at SGD step $t$\\
		\hline
		$\mnoise$ & additive noise probability measure on $\R^{\dimp}$\\
		\hline
		$B_{\|\cdot\|}(x, R), \oll{B}_{\|\cdot\|}(x, R)$ & open (resp. closed) ball of centre $x$ and radius $R$ for the norm $\|\cdot\|$ \\
		\hline
	\end{tabular}
	\label{tab:notations}
\end{table}

\redtwo{\section{Postponed Proofs}\label{sec:proof_SW_gamma}}

\paragraph{Proof of \ref{prop:SW_Gamma}} 

\begin{proof}
	Let $\mp \ll \bblambda$ and $B \in \mathcal{B}(\R^\dimp)$ such that $\bblambda(B) = 0$. We have, with $\lr' := 2\lr/\npoints,\;\data := (X, Y, \theta),\; \mdata := \mxn \otimes \myn \otimes \bbsigma$ and $\Data := \Xn \times \Yn \times \SS^{\dimy-1},$	
	\begin{align*}\mp P_\lr(B) &= \Int{\R^\dimp\times \Data}{}\mathbbold{1}_B\left[\p -\lr'\Sum{k=1}{\npoints}\left(\dr{\p}{}{T}(\p, x_k)\right)^\top  \theta \theta^\top (T(\p, x_k) - y_{\sigma_\theta^{T(\p, X), Y}(k)})\right]\dd\mp(\p)\dd\mdata(\data) \\
		&\leq \Sum{\tau \in \mathfrak{S}_\npoints}{}\Int{\Data}{}I_\tau(\data)\dd\mdata(\data),
	\end{align*}	
	where $I_\tau(\data) := \Int{\R^\dimp}{}\mathbbold{1}_B\left(\phi_{\tau,\data}(\p)\right)\dd\mp(\p)$, with $\phi_{\tau,\data} := \p - \lr'\underbrace{\Sum{k=1}{\npoints}\left(\dr{\p}{}{T}(\p, x_k)\right)^\top  \theta \theta^\top (T(\p, x_k) - y_{\tau(k)})}_{\psi_{\tau,\data} := }$.
	
	Let $\tau \in \mathfrak{S}_\npoints$ and $(X, Y, \theta) \in \Data$. Using \ref{ass:C2_ae}, separate $I_\tau(\data) = \Sum{j\in J}{}\Int{\U_j(X)}{}\mathbbold{1}_B\left(\p - \psi_{\tau,\data}(\p)\right)\dd\mp(\p),$ where the differentiability structure $(\U_j(X))_{j \in J(X)}$ is obtained using the respective differentiability structures: for each $k \in \llbracket 1, \npoints \rrbracket$, \ref{ass:C2_ae} yields a structure $(\U_{j_k}(x_k))_{j_k \in J_k(x_k)}$ of $\p \longmapsto T(\p, x_k)$, which depends on $x_k$, hence the $k$ indices.
	
	To be precise, define for $j = (j_1, \cdots, j_\npoints) \in J_1(x_1) \times \cdots \times J_\npoints(x_\npoints),\; \U_j(X) := \Inter{k=1}{\npoints}\U_{j_k}(x_k)$, and $J(X) := \left\lbrace (j_1, \cdots, j_\npoints) \in J_1(x_1) \times \cdots \times J_\npoints(x_\npoints)\ |\ \U_j(X) \neq \varnothing \right\rbrace$. In particular, for any $k \in \llbracket 1, \npoints \rrbracket,\; T(\cdot, x_k)$ is $\mathcal{C}^2$ on $\U_j(X)$. Notice that the derivatives are not necessarily defined on the border $\partial \U_j(X)$, which is of Lebesgue measure 0 by \ref{ass:C2_ae}, thus the values of the derivatives on the border do not change the value of the integrals (the integrals may have the value $+\infty$, depending on the behaviour of $\phi_{\tau, s}$, but we shall see that they are all finite when $\lr$ is small enough).
	
	We drop the $\data, \tau$ index in the notation, and focus on the properties of $\phi$ and $\psi$ as functions of $\p$. Our first objective is to determine a constant $K>0$, independent of $\p, \data, \tau$, such that $\psi$ is $K$-Lipschitz on $\U_j(X)$. 
	
	First, let $\chi := \app{\U_j(X)}{\R^\dimp}{\p}{\left(\dr{\p}{}{T}(\p, x_k)\right)^\top  \theta \theta^\top  T(\p, x_k)}$. The function $\chi$ is of class $\mathcal{C}^1$, therefore we determine its Lipschitz constant by upper-bounding the $\|\cdot\|_2$-induced operator norm of its differential, denoted by $\nt{\cfrac{\partial\chi}{\partial\p}(\p)}_2$. Notice that $\chi(\p) = \cfrac{1}{2}\cfrac{\partial}{\partial \p}\left(\theta^\top  T(\p, x_k)\right)^2$.
	
	Now $\nt{\cfrac{\partial^2}{\partial\p^2}\left(\theta^\top  T(\p, x_k)\right)^2}_2 \leq \dimp\underset{(i_1, i_2) \in \llbracket 1, \dimp \rrbracket^2}{\max}\ \left|\cfrac{\partial^2}{\partial\p_{i_1}\partial\p_{i_2}}\left(\theta^\top  T(\p, x_k)\right)^2\right|$, using \ref{ass:d2T2_and_dT_bounded} and $|\theta_i| \leq 1$,
	$$\left|\cfrac{\partial^2}{\partial\p_{i_1}\partial\p_{i_2}}\left(\theta^\top  T(\p, x_k)\right)^2\right| \leq \Sum{(i_3, i_4) \in \llbracket 1, \dimy \rrbracket^2}{}\left|\theta_{i_3}\theta_{i_4}\cfrac{\partial^2 }{\partial\p_{i_1}\partial\p_{i_2}}\Big([T(\p, x_k)]_{i_3}[T(\p, x_k)]_{i_4}\Big)\right| \leq \dimy^2 \MddT.$$	
	We obtain that $\chi$ is $\frac{1}{2}\dimp\dimy^2\MddT$-Lipschitz. 
	
	Second, let $\omega:  \p \in \U_j(X) \longmapsto \left(\dr{\p}{}{T}(\p, x_k)\right)^\top  \theta \theta^\top y_{\tau(k)}$, also of class $\mathcal{C}^1$. We re-write $\left[\cfrac{\partial \omega}{\partial\p}(\p)\right]_{i_1, i_2} = y_{\tau(k)}^\top \theta\theta^\top \cfrac{\partial^2 T}{\partial \p_{i_1}\partial\p_{i_2}}(\p, x_k)$, and conclude similarly by \ref{ass:d2T2_and_dT_bounded} that $\omega$ is $\|y_{\tau(k)}\|_2\dimp \MddT$-Lipschitz.
	
	Finally, $\psi = \Sum{k=1}{\npoints}(\chi_k - \omega_k)$, and is therefore $K := (\frac{1}{2}\dimy^2 + R_y )\dimp\npoints\MddT$-Lipschitz, with $R_y$ from \ref{ass:mx_my}. We have proven that $\nt{\cfrac{\partial \psi}{\partial \p}(\p)}_2 \leq K$ for any $\p \in \U_j(X)$, and that $K$ does not depend on $X, Y, \theta, j$ or $\p$.
	
	We now suppose that $\lr' < \frac{1}{K}$, which is to say $\lr < \frac{\npoints}{2K}$. Under this condition, $\phi: \U_j(X) \longrightarrow \R^\dimp$ is injective. Indeed, if $\phi(\p_1) = \phi(\p_2)$, then $\|\p_1-\p_2\|_2 = \lr'\|\psi(\p_1)-\psi(\p_2)\|_2 \leq \lr'K\|\p_1-\p_2\|_2$, thus $\p_1 = \p_2$. Furthermore, for any $\p \in \U_j(X),\; \cfrac{\partial \phi}{\partial \p}(\p) = \mathrm{Id}_{\R^\dimp} - \lr'\cfrac{\partial \psi}{\partial \p}(\p)$, with $\nt{\lr'\cfrac{\partial \psi}{\partial \p}(\p)}_2 < 1$, thus the matrix $\cfrac{\partial \phi}{\partial \p}(\p)$ is invertible (using the Neumann series method). By the global inverse function theorem, $\phi: \U_j(X) \longrightarrow \phi(\U_j(X))$ is a $\mathcal{C}^1$-diffeomorphism.
	
	\red{Using the change-of-variables formula, we have $\Int{\U_j(X)}{}\mathbbold{1}_B(\phi(\p))\dd\mp(\p) = \Int{\U_j(X)}{}\mathbbold{1}_B(\p')\dd\phi\#\mp(\p') = \phi\#\mp(B)$, we have now shown that $\phi$ is a $\mathcal{C}^1$-diffeomorphism, thus since $\mp \ll \bblambda$, $\phi\#\mp \ll \bblambda$. ($\alpha \ll\beta$ denoting that $\alpha$ is absolutely continuous with respect to $\beta$). Since $\bblambda(B) = 0$, it follows that the integral is 0, then by sum over $j$, $I_\tau(\data) = 0$ and finally $\mp P_\lr(B) = 0$ by integration over $\data$ and sum over $\tau$.}
\end{proof}

\section{Background on Non-Smooth and Non-Convex Analysis}\label{sec:nonsmooth}

This work is placed within the context of non-smooth optimisation, a field of study in part introduced by Clarke with the so-called Clarke differential, which we introduced in Equation \ref{eqn:clarke} (see \citep{clarke1990optimization} for a general reference on this object). The purpose of this appendix is to present several adjacent objects that can be useful to the application of our results, even though we do not need them in order to prove our theorems. 

\subsection{Conservative Fields}\label{sec:conservative_fields}

The Clarke differential $\partial_C$ of a locally Lipschitz function $g: \R^d \longrightarrow \R$ (defined in Equation \ref{eqn:clarke}) is an example of a \textit{set-valued map}. Such a map is a function $D: \R^p \rightrightarrows \R^q$ from the subsets of $\R^p$ to the subsets of $\R^q$, for instance in the case of the Clarke differential, we have the signature $\partial_C g: \R^d \rightrightarrows \R^d$. A set-valued map $D$ is \textit{graph closed} if its graph $\lbrace (\p, v)\ |\ \p \in \R^p,\; v\in D(\p) \rbrace$ is a closed set of $\R^{p+q}$. A set-valued map $D$ is said to be a \textit{conservative field}, when it is graph closed, has non-empty compact values and for any absolutely continuous loop $\gamma \in \mathcal{C}_{\mathrm{abs}}([0, 1], \R^p)$ with $\gamma(0) = \gamma(1)$, we have 
$$\Int{0}{1}\underset{v \in D(\gamma(s))}{\max}\ \langle \dot \gamma(s), v \rangle \dd s = 0. $$
Similarly to primitive functions in calculus, one may define a function $g: \R^d \longrightarrow \R$ using a conservative field $D: \R^d\rightrightarrows \R^d$ up to an additive constant through following expression:
\begin{equation}\label{eqn:potential}
	g(\p) = g(0) + \Int{0}{1}\underset{v \in D(\gamma(s))}{\max}\ \langle \dot \gamma(s), v \rangle \dd s, \quad \forall \gamma \in \mathcal{C}_{\mathrm{abs}}([0, 1], \R^p)\ \text{such\ that}\ \gamma(0) = 1\ \text{and}\ \gamma(1)=\p.
\end{equation}
In this case, we say that $g$ is a \textit{potential function} for the field $D$. This notion allows us to define a new regularity class: a function $g: \R^d \longrightarrow \R$ is called \textit{path differentiable} when there exists a conservative field of which it is a potential. A standard result in non-smooth optimisation is the following equivalence between different notions of regularity:

\begin{prop}\label{prop:equiv_path_regular}\citet{bolte2021conservative}, Corollary 2. Let $g: \R^d \longrightarrow \R$ locally Lipschitz. We have the equivalence between the following statements:
	\begin{itemize}
		\item $g$ is path differentiable
		\item $\partial_C g$ is a conservative field
		\item $g$ has a \textit{chain rule} for the Clarke differential $\partial_C$:
		\begin{equation}\label{eqn:chaine_rule}
			\forall \p \in \mathcal{C}_{\mathrm{abs}}(\R_+, \R^d), \; \ull{\forall} s > 0,\; \forall v \in \partial_Cg(\p(s)),\; v^\top \dot \p(s) = (g \circ \p)'(s).
		\end{equation}
	\end{itemize}
\end{prop}
This equivalence justifies the terminology used in \ref{cond:A5}. The reader seeking a complete presentation of conservative field theory may refer to \citep{bolte2021conservative}.

\subsection{Conservative Mappings}\label{sec:conservative_mappings}

The notion of conservative fields for real-valued locally Lipschitz functions $g: \R^d \longrightarrow \R$ can be generalised to \textit{conservative mappings} for vector-valued locally Lipschitz functions $g: \R^p \longrightarrow \R^q$, which one may see as a generalised Jacobian matrix (see \citep{bolte2021conservative}, Section 3.3 for further details). A set-valued map $J: \R^p \rightrightarrows \R^{q\times p}$ is a conservative mapping for such a $g$ if
\begin{equation}\label{eqn:conservative_mappings}
	\forall \p \in \mathcal{C}_{\mathrm{abs}}(\R_+, \R^p),\; \ull{\forall} s > 0,\; (g\circ \p)'(s) = M\dot \p(t),\; \forall M \in  J(\p(s)).
\end{equation}
In this case, we shall say that $g$ is path differentiable. Note that if each coordinate function $g_i$ is the potential of a conservative field $D_i$, then the set-valued map 
$$J(\p) = \left\lbrace \left(\begin{array}{c}
	v_1^\top\\
	\vdots\\
	v_q^\top\\
\end{array}\right)\ :\ \forall i \in \llbracket 1, q \rrbracket,\; v_i \in D_i(\p) \right\rbrace$$ 
is a conservative mapping for $g$ (although not all conservative mappings for $g$ can be written in this manner). As a consequence, one could interpret (simplistically) vector-valued path differentiability as coordinate-wise path differentiability.

\subsection{Clarke Regularity}\label{sec:Clarke_regular}

Another notion of regularity for locally Lipschitz functions is that of \textit{Clarke regularity}. Let $g: \R^p \longrightarrow \R^q$ and $\p \in \R^p$, $g$ is said to be \textit{Clarke regular} at $\p$ if the two quantities
$$g^\circ(\p; v) := \underset{\substack{\p' \rightarrow \p \\ t \searrow 0}}{\operatorname{limsup}} \cfrac{g(\p'+tv) - g(\p')}{t} \quad \text{and} \quad g'(\p;v) := \underset{t \searrow 0}{\operatorname{lim}}\cfrac{g(\p+tv)-g(\p)}{t}$$
exist and are equal for all $v \in \R^p$. Note that this notion implies path differentiability by \citep{bolte2021conservative}, Proposition 2. Clarke regularity is the central concept of Clarke's monograph \citep{clarke1990optimization}.

\redtwo{\subsection{Semi-Algebraic Functions}\label{sec:semi_algebraic}}

In non-smooth analysis, one of the simplest regularity cases is the class of \textit{semi-algebraic} functions, which are essentially piecewise polynomial functions defined on polynomial pieces. To be precise, a set $\mathcal{A} \subset \R^d$ is \textit{semi-algebraic} if it can be written under the form
$$\mathcal{A} = \Reu{i=1}{n}\Inter{j=1}{m}\left\lbrace \p \in \R^d\ |\ P_{i,j}(\p) < 0,\; Q_{i,j}(\p) = 0 \right\rbrace,$$
where the $P_{i,j}$ and $Q_{i,j}$ are real multivariate polynomials. A function $g: \R^p \longrightarrow \R^q$ is \textit{semi-algebraic} if its graph $\mathcal{G} := \lbrace (\p, g(\p))\ |\ \p \in \R^p \rbrace$ is semi-algebraic.

A locally Lipschitz real-valued semi-algebraic function is path differentiable (see for instance \citep{bolte2021conservative}, Proposition 2), and in the light of \citep{bolte2021conservative}, Lemma 3, this is also the case in the vector-valued case. Another useful property of semi-algebraic functions is that their class is stable by composition and product. The interested reader may consult \citep{Wakabayashi_semialgebraic} for additional properties of semi-algebraic objects, or \citep{coste1999introduction, van1996geometric}, for a presentation of o-minimal structures, a generalisation of this concept.

\redtwo{\section{Suitable Neural Networks}\label{sec:suitable_NNs}}

In this section, we detail our claim that typical NN structures satisfy our conditions. To this end, we define a class of practical neural networks whose properties are sufficient (not all NNs that satisfy our assumptions are within this framework). Consider $\classT$ the set of NNs $T$ of the form
$$T: \app{\R^{\dimp}\times \R^{\dimx}}{\R^\dimy}{(\p,x)}{\widetilde{T}(\p, x)\mathbbold{1}^\varepsilon_{B(0, R_\p)}(\p)\mathbbold{1}^\varepsilon_{B(0, R_x)}(x)},$$
with $R_\p, R_x > 0$ and $\varepsilon > 0$. The function $\mathbbold{1}^\varepsilon_{B(0, R)}$ is a smoothed version of the usual indicator function $\mathbbold{1}_{B(0, R)}$: it is any function that has value 1 in $B(0, R - \varepsilon)$, 0 outside $B(0, R + \varepsilon)$ and is $\mathcal{C}^2$-smooth (see \ref{rem:smooth_indicator} for a possible construction). Given that one may take arbitrarily large radii, these indicators are added for theoretical purposes and impose no realistic constraints in practice. Additionally, $\widetilde{T} = h_N$, the $\nlayers$-th layer of a recursive NN structure defined by
$$\layerout_0(\p, x) = x, \quad \forall n \in \llbracket 1, \nlayers \rrbracket,\; \layerout_n = \app{\R^{\dimp}\times \R^{\dimx}}{\R^{d_n}}{(\p,x)}{\activ_n\left(\Sum{i=0}{n-1}\linunit_{n,i}(\p)\layerout_i(\p,x) + \intercept_n \p \right)},$$
where:
\begin{itemize}
	\item \textit{All} functions $\activ_n: \R \longrightarrow \R$ are $\mathcal{C}^2$-smooth, or \textit{all} locally Lipschitz semi-algebraic activation functions (applied entry-wise). The former condition is satisfied by the common sigmoid, hyperbolic tangent or softplus activations. The latter condition applies to the non-differentiable ReLU activation, its "Leaky ReLU" extension, and continuous piecewise polynomial activations. Note that other non-linearities such as softmax can also be considered under the same regularity restrictions, but we limit ourselves to entry-wise non-linearities for notational consistency.
	\item Each dimension $d_n$ is a positive integer, with obviously $d_{\nlayers} = \dimy$, the output dimension.
	\item Each $\linunit_{n, i}$ is a linear map: $\R^\dimp \longrightarrow \R^{d_n \times d_i}$, which maps a parameter vector $\p$ to a $d_n \times d_i$ matrix. Since the entire parameter vector $\p$ is given at each layer, this allows the architecture to only use certain parameters at each layer (as is more typical in practice). One may see this map as a 3-tensor of shape $(d_n, d_i, \dimp)$, as specified in the formulation
	\begin{equation}\label{eqn:A_tensor_form}
		\forall \p \in \R^\dimp,\; \forall \layerout \in \R^{d_i},\; \linunit_{n, i}(u)\layerout = \left(\Sum{j_2=1}{d_i}\Sum{j_3=1}{\dimp}\linunit_{j_1, j_2, j_3}^{(n, i)}\layerout_{j_2}\p_{j_3}\right)_{j_1 \in \llbracket 1, d_n \rrbracket} \in \R^{d_n}.
	\end{equation}
	\item The matrix $\intercept_n \in \R^{d_n \times \dimp}$ determines the intercept from the full parameter vector $\p$.
\end{itemize}

In this model, each layer depends on all the previous layers, allowing for residual inputs for instance. Overall, all typical networks fit this description, once bounded using the indicator functions, with only a technicality on the regularities of the activations which need to be \textit{all} $\mathcal{C}^2$-smooth, or \textit{all} semi-algebraic. One could extend this class of NNs to those with \textit{definable} activations within the same o-minimal structure (similarly to \citet{davis2020stochastic} and \citet{bolte2021conservative}).%

\begin{remark}\label{rem:smooth_indicator} We mention that we may construct a $\mathcal{C}^\infty$-smooth $\mathbbold{1}^\varepsilon_{B(0, R)}$ in $\R^d$ explicitly as follows:
	$$f(s):= \left\lbrace\begin{array}{c}
		e^{-1/s}\text{\ if\ }s>0 \\
		\text{else\ } 0
	\end{array} \right.,\quad g(s):=\cfrac{f(s)}{f(s)+f(1-s)},\quad \mathbbold{1}^\varepsilon_{B(0, R)} := \app{\R^d}{[0, 1]}{\p}{g\left(\cfrac{(R + \varepsilon)^2 - \|\p\|_2^2}{4R\varepsilon}\right)}. $$
\end{remark}

Before proving the properties of NNs from the class $\classT$, we require a technical result on path differentiable functions.

\begin{prop}\label{prop:prod_path_diff} Let $f: \R^d\longrightarrow\R$ path differentiable, and $g:\R^d\longrightarrow\R$ of class $\mathcal{C}^1$. Then their product $fg$ is path differentiable. 
\end{prop}
\begin{proof}
	Our objective is to apply \citep{bolte2021conservative} Corollary 2 (stated in \ref{prop:equiv_path_regular}), which is to say that $h := fg$ admits a chain rule for $\partial_C h$). First, we apply the definition of the Clarke differential and compute
	$$\forall \p \in \R^d,\;\partial_C f(\p) = f(\p)\nabla g(\p) + g(\p)\partial_Cf(\p) := \left\lbrace f(\p)\nabla g(\p) + g(\p)v\ |\ v \in \partial_Cf(\p)\right\rbrace.$$
	Note that we used the smoothness of $g$. We now consider an absolutely continuous curve $\p \in \mathcal{C}_{\mathrm{abs}}(\R_+, \R^d)$. By \cite{bolte2021conservative} Lemma 2, since $f$ is path differentiable, $f\circ \p$ is differentiable almost everywhere. Let $D$ the associated set of differentiability, then let $s\in D$ and $v \in \partial_C h(\p(s))$, writing $v = f(\p(s))\nabla g(\p(s)) + g(\p(s))w$ with $w \in \partial_Cf(\p(s))$. We compute $(h\circ\p)'(s) = (f\circ \p)'(s)g(\p(s)) + f(\p(s))(g\circ \p)'(s)$. Now since $f$ is path differentiable and $w \in \partial_Cf(\p(s))$, by \ref{prop:equiv_path_regular} item 3, we have $(f\circ \p)'(s) = \langle w, \dot \p(s)\rangle$. On the other hand, $(g\circ \p)'(s) = \langle \nabla g(\p(s)), \dot \p(s)\rangle$ since $g$ is $\mathcal{C}^1$. Finally by definition of $v$ and bilinearity of $\langle \cdot, \cdot \rangle$,
	$$(h\circ\p)'(s) = \langle w, \dot \p(s)\rangle g(\p(s)) + f(\p(s))\langle \nabla g(\p(s)), \dot \p(s)\rangle = \langle v, \dot \p(s)\rangle.$$
\end{proof}
We now have all the tools to prove that the class of NNs $\classT$ satisfies all of the assumptions of our paper.

\begin{prop} All networks of the class $\classT$ verify \ref{ass:C2_ae}, \ref{ass:T_loclip}, \ref{ass:T_slow_increase}, \ref{ass:d2T2_and_dT_bounded} and \ref{ass:T_path_diff}.
\end{prop}
\begin{proof}
	Let $T \in \classT$, and $\widetilde{T}$ its associated underlying network. We begin with regularity considerations. 
	
	\paragraph{Verifying Assumptions 1 and 7 in the $\mathcal{C}^2$ Case} In the case where the activations are $\mathcal{C}^2$-smooth, then each $\widetilde{T}(\cdot, x)$ is also of class $\mathcal{C}^2$. Furthermore, the smooth indicator $\mathbbold{1}^\varepsilon_{B(0, R_\p)}$ is $\mathcal{C}^\infty$-smooth, thus we can conclude that $T(\cdot, x)$ is $\mathcal{C}^2$-smooth, and thus satisfies \ref{ass:C2_ae} trivially. In particular, $T(\cdot, x)$ is path differentiable for any $x \in \R^\dimx$, thus $T$ also satisfies \ref{ass:T_path_diff}.
	
	\paragraph{Verifying Assumptions 1 and 7 in the Semi-Algebraic Case} In the case where the activations are locally Lipschitz and semi-algebraic, it follows that each $\widetilde{T}(\cdot, x)$ is semi-algebraic, which yields naturally a differentiability structure associated to the polynomial pieces, satisfying \ref{ass:C2_ae}. Furthermore, this regularity yields path differentiability by \citep{bolte2021conservative}, Proposition 2. By product with the smooth indicator function, $T$ is path differentiable by \ref{prop:prod_path_diff}, therefore it satisfies \ref{ass:T_path_diff} .
	
	\paragraph{Verifying Assumption 2 in the $\mathcal{C}^2$ Case} In the case where the activations are $\mathcal{C}^2$-smooth, it is clear that by composition and product $(\p, x) \longmapsto \widetilde{T}(\p, x)$ is \textit{jointly} $\mathcal{C}^2$-smooth. As a consequence, it is Lipschitz jointly in $(\p, x)$ on any compact of $\R^{\dimp}\times\R^{\dimy}$, and by product with the smooth indicators, so is $T$. Since $T$ is zero outside $\oll{B}(0, R_\p+\varepsilon)\times\oll{B}(0, R_x+\varepsilon)$, we conclude that it is globally Lipschitz in $(\p, x)$.
	
	\paragraph{Verifying Assumption 2 in the Semi-Algebraic Case} In the case of locally Lipschitz and semi-algebraic activations, we prove that $\widetilde{T}$ is jointly Lipschitz on any compact $\mathcal{K}$ by strong induction on $n \in \llbracket 1, \nlayers \rrbracket$. Let $\mathcal{K} = \mathcal{K}_1 \times \mathcal{K}_2$ a product compact of $\R^{\dimp}\times\R^{\dimy}$, and $P_n:$ "$\exists \lipT_n >0: \layerout_n$ is $\lipT_n$-Lipschitz on $\mathcal{K}$". The initialisation $P_0$ is trivial, since $z(\p, x) = x$. Let $n \in \llbracket 1, \nlayers \rrbracket$ and assume $P_i$ to hold true for $i \in \llbracket 0, n-1 \rrbracket$. In particular, the $\layerout_i$ are jointly continuous in $(\p, x)$, allowing the definition of 
	$$M := \underset{(\p, x) \in \mathcal{K}}{\max}\ \left|\Sum{i=0}{n-1}\linunit_{n,i}(\p)\layerout_i(\p,x) + \intercept_n \p\right|.$$
	Since $\activ_n$ is locally Lipschitz, a covering argument shows that there exists $\lipT_{\activ_n}>0$ such that $\activ_n$ is $\lipT_{\activ_n}$-Lipschitz on $[-M, M]$. Now let $(\p_1, \p_2) \in \mathcal{K}_1^2$ and $(x_1, x_2) \in \mathcal{K}_2^2$. We have
	\begin{align}\label{eqn:show_T_loclip1}
		\|\layerout_n(\p_1, x_1) - \layerout_n(\p_2, x_2)\|_2 &\leq \lipT_{\activ_n}\left\|\Sum{i=0}{n-1}\linunit_{n,i}(\p_1)\layerout_i(\p_1,x_1) + \intercept_n \p_1 - \Sum{i=0}{n-1}\linunit_{n,i}(\p_2)\layerout_i(\p_2,x_2) - \intercept_n \p_2\right\|_2 \nonumber\\
		&\leq \lipT_{\activ_n} \left(\opn{\intercept_n}\|\p_1 - \p_2\|_2 + \Sum{i=0}{n-1}\left\|\linunit_{n,i}(\p_1)\layerout_i(\p_1,x_1) - \linunit_{n,i}(\p_2)\layerout_i(\p_2,x_2)\right\|_2\right),
	\end{align}
	where $\opn{\cdot}$ denotes the $\|\cdot\|_2$-induced operator norm. Let $i \in \llbracket 0, n-1 \rrbracket$, we separate the norm:
	\begin{align}\label{eqn:show_T_loclip2}
		\left\|\linunit_{n,i}(\p_1)\layerout_i(\p_1,x_1) - \linunit_{n,i}(\p_2)\layerout_i(\p_2,x_2)\right\|_2 & \leq \left\|\linunit_{n,i}(\p_1)\layerout_i(\p_1,x_1) - \linunit_{n,i}(\p_2)\layerout_i(\p_1,x_1)\right\|_2 =: \Delta_1 \nonumber\\
		& + \left\|\linunit_{n,i}(\p_2)\layerout_i(\p_1,x_1) - \linunit_{n,i}(\p_2)\layerout_i(\p_2,x_2)\right\|_2 =: \Delta_2.
	\end{align}
	For $\Delta_1$, use the tensor form \ref{eqn:A_tensor_form} and the inequality $\|x\|_2 \leq \sqrt{d}\|x\|_\infty$ for $x\in \R^d$, then $\|\p\|_\infty \leq \|\p\|_2$:
	\begin{align}\label{eqn:show_T_loclip3}
		\Delta_1 &\leq \sqrt{d_n}\left\|\left(\Sum{j_2=1}{d_i}\Sum{j_3=1}{\dimp}\linunit_{j_1, j_2, j_3}^{(n, i)}\layerout_i(u_1,x_1)_{j_2}(\p_{j_3}^{(1)}-\p_{j_3}^{(2)})\right)_{j_1 \in \llbracket 1, d_n \rrbracket}\right\|_\infty \nonumber\\
		&\leq \sqrt{d_n}\underset{j_1,j_2,j_3}{\max}\ |A_{j_1,j_2,j_3}^{(n,i)}|\ \underset{(\p,x) \in \mathcal{K}_1\times\mathcal{K}_2}{\max}\ \|\layerout_i(\p,x)\|_\infty\ \|\p_1 - \p_2\|_\infty \nonumber\\
		&\leq \lipT_{\Delta_1}\|\p_1 - \p_2\|_2.
	\end{align}
	For $\Delta_2$, we leverage $P_i$ and obtain
	\begin{equation}\label{eqn:show_T_loclip4}
		\Delta_2 \leq \underset{\p \in \mathcal{K}_1}{\max}\ \opn{\linunit_i(\p)}\ \|\layerout_i(\p_1, x_1) - \layerout_i(\p_2, x_2)\|_2 \leq \underset{\p \in \mathcal{K}_1}{\max}\ \opn{\linunit_i(\p)}\ \lipT_i\left(\|\p_1 - \p_2\|_2 + \|x_1 - x_2\|_2\right).
	\end{equation}
	
	Combining \ref{eqn:show_T_loclip1} \ref{eqn:show_T_loclip2} \ref{eqn:show_T_loclip3} and \ref{eqn:show_T_loclip4} shows $P_n$ and concludes the induction, which in turn shows that $\widetilde{T}$ is jointly Lipschitz on any compact. Like in the smooth case, we conclude that $T$ is globally Lipschitz, and thus that \ref{ass:T_loclip} holds.
	
	\paragraph{Verifying Assumption 4} Under both cases of regularity for the activations, 
	$$g := x \longmapsto \underset{\p \in \oll{B}(0, R_\p + \varepsilon)}{\max}\ \|\widetilde{T}(\p, x)\|_2\mathbbold{1}^\varepsilon_{B(0, R_x)}(x)$$
	is measurable and bounded. Furthermore, observe that for $\p, x \in \R^\dimp \times \R^\dimx,\; \|T(\p,x)\|_2 \leq g(x)$. As a consequence, \ref{ass:T_slow_increase} holds. 
	
	\paragraph{Verifying Assumption 5 in the $\mathcal{C}^2$ case} If all activations are $\mathcal{C}^2$-smooth, both $\widetilde{T}$ and its coordinate-wise products $T_i\times T_j$ are $\mathcal{C}^2$-smooth jointly in $(\p, x)$. As a result, one may bound these terms on $(\p, x) \in \oll{B}(0, R_\p + \varepsilon) \times \oll{B}(0, R_x+\varepsilon)$ by a constant $M$, independent of $\p, x$, and the assumption is verified. 
	
	\paragraph{Verifying Assumption 5 in the semi-algebraic case} In the semi-algebraic case, there exists a structure $(\mathcal{U}_j)_{j \in J}$ of open sets of $\R^\dimp \times \R^\dimx$ whose closures cover the entire space, such that $\widetilde{T}$ is polynomial in $(\p, x)$ on each $\mathcal{U}_j$, with $J$ finite (this is possible since $\widetilde{T}$ is jointly semi-algebraic). The NN can be written $T(\p, x) = \widetilde{T}(\p, x)\mathbbold{1}^\varepsilon_{B(0, R_\p)}(\p)\mathbbold{1}^\varepsilon_{B(0, R_x)}(x)$, and is therefore $\mathcal{C}^2$-smooth on each $\mathcal{U}_j$. Furthermore, its restriction to $\mathcal{U}_j$ is extendable $\mathcal{C}^2$-smoothly to $\oll{\mathcal{U}_j}$ (we shall not introduce a different notation to these extensions, for legibility). As a result, one may introduce the following bounds on the derivatives of the coordinate functions on the intersection of the compact $\mathcal{K} := \oll{B}(0, R_\p+\varepsilon)\times\oll{B}(0, R_x+\varepsilon)$ and $\oll{\mathcal{U}_j}$: there exists an $M_j>0$ such that
	$$\forall (\p,x) \in \mathcal{K}\ \cap\ \oll{\mathcal{U}_j},\; \left|\cfrac{\partial^2}{\partial \p_{i_1} \partial \p_{i_2}} \Big([T(\p, x)]_{i_3} [T(\p, x)]_{i_4}\Big)\right| \leq M_j\; \mathrm{and}\;\left\|\cfrac{\partial^2 T}{\partial \p_{i_1}\partial\p_{i_2}}(\p, x)\right\|_2 \leq M_j.$$
	Since $J$ is finite and the $(\mathcal{U}_j)_{j\in J}$ cover $\mathcal{K}$, we deduce that this bound holds for $(\p,x) \in \mathcal{K}$ for a common constant $M>0$. Moreover, since $T$ is the zero function outside of $\mathcal{K}$, this bounds also holds for any $(\p, x) \in \R^\dimp\times\R^\dimx$. Finally, this shows that \ref{ass:d2T2_and_dT_bounded} holds.	
\end{proof}

\redtwo{\section{Generalisation to Other Sliced Wasserstein Orders}\label{sec:other_p}}

In this section, we shall discuss how some of our results can be extended by replacing the 2-SW term $\SW_2^2$ with $\SW_p^p$ for $p\in [1, +\infty[$.

\paragraph{Determining Lipschitz Constants} The first difficulty lies in showing that the functions $w_{\theta}^{(p)}:=(X, Y)\longmapsto\W_p^p(P_\theta\#\bbgamma_X, P_\theta\#\bbgamma_Y)$ still have a locally Lipschitz regularity similar to \ref{prop:w_unif_locLip} (this proposition is only shown for $p=2$ in \citep{discrete_sliced_loss}). We generalise their result in the following proposition.

\begin{prop}\label{prop:wp_unif_loc_lip} Let $K^{(p)}_w(r, X, Y) := p\npoints(r + \|X\|_{\infty, 2} + \|Y\|_{\infty, 2})^{p-1}$, for $X, Y \in \R^{\npoints \times \dimy}$ and $r>0$. Then $w^{(p)}_\theta(\cdot, Y)$ is $K^{(p)}_w(r, X, Y)$-Lipschitz in the neighbourhood $B_{\|\cdot\|_{\infty, 2}}(X, r)$:
	$$\forall Y', Y'' \in B_{\|\cdot\|_{\infty, 2}}(X, r),\; \forall \theta \in \SS^{\dimy-1},\; |w_\theta(Y', Y) - w_\theta(Y'', Y)| \leq K^{(p)}_w(r, X, Y) \|Y'-Y''\|_{\infty, 2}.$$
\end{prop}
\begin{proof}
	Let $X, Y \in \R^{\npoints \times \dimy}, r>0$ and $Y', Y'' \in
	B_{\|\cdot\|_{\infty, 2}}(X, r)$. By \citep{discrete_sliced_loss} Lemma 2.1,
	we have $|w_\theta^{(p)}(Y')-w_\theta^{(p)}(Y'')| \leq \|C'-C''\|_F$, where
	$\|\cdot\|_F$ denotes the Frobenius norm, and $C'$ is a $n\times n$ matrix
	of entries $C'_{k,l} = |\theta^\top y_k' - \theta^\top y_l|^p$, with
	similarly $C''_{k,l} = |\theta^\top y_k'' - \theta^\top y_l|^p$. Now
	consider the function
	$$g_{y_l}:= \app{\R^\dimy}{\R}{y}{|\theta^\top y - \theta^\top y_l|^p},$$
	which satisfies $C_{k,l}' = g_{y_l}(y_k')$, and is differentiable almost-everywhere, with $\nabla g_{y_l}(y) = p|\theta^\top y - \theta^\top y_l|^{p-1}\theta$. For almost every $y \in B(x_k, r)$, we have
	\begin{align*}
		\|\nabla g_{y_l}(y)\|_2 &\leq p\|y-y_l\|_2^{p-1} = p \|y-x_k + x_k -y_l\|_2^{p-1} \\ &\leq p \left(\|y-x_k\|_2 + \|x_k\|_2+ \|y_l\|_2\right)^{p-1}\leq p(r+\|X\|_{\infty, 2}+\|Y\|_{\infty, 2})^{p-1}.
	\end{align*}	
	As a result, $g_{y_l}$ is $p(r+\|X\|_{\infty, 2}+\|Y\|_{\infty, 2})^{p-1}$-Lipschitz in $B(x_k, r)$. Now since $Y', y''\in B_{\|\cdot\|_{\infty, 2}}(X, r)$, we have $y_k', y_k'' \in B(x_k, r)$, thus
	$$|[C']_{k,l} - [C'']_{k,l}| = |g_{y_l}(y_k') - g_{y_l}(y_k'')| \leq p(r+\|X\|_{\infty, 2}+\|Y\|_{\infty, 2})^{p-1} \|y_k' - y_k''\|_2.$$	
	Then $\|C'-C''\|_F  = \sqrt{\sum_{k,l}|[C']_{k,l} - [C'']_{k,l}|^2} \leq np(r+\|X\|_{\infty, 2}+\|Y\|_{\infty, 2})^{p-1} \|Y' - Y''\|_{\infty, 2}.$	
\end{proof}

Our results regarding the local Lipschitz property of $f$ and $F$ adapt immediately using the same method with the different constant $K^{(p)}_w(r, X, Y)$, we obtain the following constant for $f$ (with $\lipT$ from \ref{ass:T_loclip}):
$$K_f^{(p)}(\varepsilon, \p_0, X, Y) = pn\lipT\left(\varepsilon \lipT + \|T(\p_0, X)\|_{\infty, 2} + \|Y\|_{\infty, 2}\right)^{p-1},$$
then the following constant for $F$:
$$K_F^{(p)}(\varepsilon, \p_0) = pn\lipT\Int{\Xn\times\Yn}{}\left(\varepsilon \lipT + \|T(\p_0, X)\|_{\infty, 2} + \|Y\|_{\infty, 2}\right)^{p-1}\dd\mxn(X)\dd\myn(Y).$$
In order to satisfy \ref{cond:A2} item i) in the case $p\neq 2$, one needs to modify \ref{ass:T_slow_increase} to require $\|T(\p, x)\|_2 \leq g(x)^{1/(p-1)}(1+\|\p\|_2)^{1/(p-1)}$, which in realistic cases is not much more expensive than asking for $T$ to be bounded, which is a property of the class of NNs that we present in \ref{sec:suitable_NNs}. 

\paragraph{Almost-Everywhere Gradient} A second difficulty lies in defining an almost-everywhere gradient $f$, since in our main text we rely on the formulation of an almost-everywhere gradient of $w_\theta^{(2)}(\cdot, Y)$ which was derived only for $p=2$ by \citet{bonneel2015sliced} and \citet{discrete_sliced_loss}. In fact, for $\theta, Y$ fixed $w_\theta^{(p)}(X, Y)$ is piecewise smooth, like $w_\theta^{(2)}(\cdot, Y)$ is piecewise quadratic. As a result, one may show that the following is an almost-everywhere gradient of $w_\theta^{(p)}(\cdot, Y)$:
$$\dr{X}{}{w^{(p)}_\theta}(X, Y) = \left(\cfrac{p}{\npoints}\sign\left(\theta^\top x_k - \theta^\top y_{\sigma_\theta^{X, Y}(k)}\right)\left|\theta^\top x_k - \theta^\top y_{\sigma_\theta^{X, Y}(k)}\right|^{p-1}\theta\right)_{k \in \llbracket 1, \npoints \rrbracket} \in \R^{\npoints \times \dimy}.$$
The chain rule now yields the following almost-everywhere gradient for $f$:
$$\varphi (\p, X, Y, \theta) = \Sum{k=1}{\npoints} \cfrac{p}{\npoints}\sign\left(\theta^\top T(\p, x_k) - \theta^\top y_{\sigma_\theta^{T(\p, X), Y}(k)}\right) \left|\theta^\top T(\p, x_k) - \theta^\top y_{\sigma_\theta^{T(\p, X), Y}(k)}\right|^{p-1} \dr{\p}{}{T}(\p, x_k)\theta.$$
\paragraph{Adapting Proposition 4} Moving on to adapting \ref{prop:SW_Gamma}, the general case $p\neq 2$ makes things substantially more technical, but one may still show that the $\psi$ functions are Lipschitz using restrictions on $T$ its first and second-order derivatives (which can be formulated in a more technical version of \ref{ass:d2T2_and_dT_bounded}). In conclusion, \ref{prop:SW_Gamma} can be adapted to apply to $p\in [1, +\infty[$, and it follows that \ref{thm:SGD_interpolated_cv} also generalises to this case. 

\paragraph{Path Differentiability} Regarding the results from \ref{sec:noised_proj_sgd}, the only substantial difference lies in showing that $T(\cdot, x)$ is path differentiable. The only missing link in the composition chain is the path differentiability of $\SWY^{(p)}:= X \longmapsto \int_{\SS^{d-1}}w_\theta^{(p)}(X, Y)\dd\bbsigma(\theta)$. In the case $p=2$, the difficulty of the integral can be circumvented by noticing that $\SWY$ is semi-concave \citep{discrete_sliced_loss}, Proposition 2.4, which implies path differentiability. This argument does not generalise to $p\in [1, +\infty[$ naturally, hence our \ref{thm:SGD_projected_noised} only generalises to $p\in [1, +\infty[$ under the conjecture that $\SWY^{(p)}$ is indeed path differentiable.

\end{document}